\documentclass[10pt,twocolumn,letterpaper]{article}

\usepackage{cvpr}              

\newtheorem{assumption}{Assumption}
\newtheorem{theorem}{Theorem}
\newtheorem{lemma}{Lemma}
\newtheorem{proof}{Proof}\usepackage{algorithm}
\usepackage{algorithmic}
\usepackage{subcaption}
\usepackage{multirow}
\usepackage{xcolor}   
\usepackage{subscript}
\definecolor{cvprblue}{rgb}{0.21,0.49,0.74}
\usepackage[pagebackref,breaklinks,colorlinks,allcolors=cvprblue]{hyperref}

\def\paperID{9574} 
\def\confName{CVPR}
\def\confYear{2025}

\title{Filter, Obstruct and Dilute: Defending Against Backdoor Attacks on Semi-Supervised Learning}

\author{
Xinrui Wang$^{1,2}$ \quad Chuanxing Geng$^{1,2}$ \quad Wenhai Wan$^{3}$ \quad Shao-yuan Li$^{1,2}$\quad Songcan Chen$^{1,2}$\\
\small $^{1}$College of Computer Science and Technology, Nanjing University of Aeronautics and Astronautics\\
\small $^{2}$MIIT Key Laboratory of Pattern Analysis and Machine Intelligence\\
\small $^{3}$ School of Computer Science and Technology, Huazhong University of Science and Technology
}

\begin{document}
\maketitle

\begin{abstract}
    Recent studies have verified that semi-supervised learning (SSL) is vulnerable to data poisoning backdoor attacks. Even a tiny fraction of contaminated training data is sufficient for adversaries to manipulate up to 90\% of the test outputs in existing SSL methods. Given the emerging threat of backdoor attacks designed for SSL, this work aims to protect SSL against such risks, marking it as one of the few known efforts in this area. Specifically, we begin by identifying that the spurious correlations between the backdoor triggers and the target class implanted by adversaries are the primary cause of manipulated model predictions during the test phase. To disrupt these correlations, we utilize three key techniques: Gaussian Filter, complementary learning and trigger mix-up, which collectively filter, obstruct and dilute the influence of backdoor attacks in both data pre-processing and feature learning. Experimental results demonstrate that our proposed method, Backdoor Invalidator (BI), significantly reduces the average attack success rate from 84.7\% to 1.8\% across different state-of-the-art backdoor attacks. It is also worth mentioning that BI does not sacrifice accuracy on clean data and is supported by a theoretical guarantee of its generalization capability. 
\end{abstract}
\vspace{-5pt}


\section{Introduction} \label{intro}
Semi-supervised learning (SSL) has made strides in leveraging small amounts of labeled data with abundant unlabeled data, showing potential for practical applications by reducing the need for extensive manual annotation\cite{learning2006semi}. However, recent studies have revealed that existing SSL methods are highly susceptible to specific types of data poisoning backdoor attacks. Adversaries can maliciously manipulate the predictions of the attacked model in the test phase by injecting a backdoor trigger (i.e., a particular pattern like small white patches on some specific position or certain kinds of noise) into a few benign images during training \cite{carlini2021poisoning, li2024survey}. As depicted in Figure \ref{framework}, this situation is even worse in SSL, where adversaries can manipulate about 90\% of the SSL model's output during inference by embedding triggers into a tiny portion of unlabeled data \cite{zeng2021rethinking}.

In contrast to the well-developed defense methods for backdoor attacks in supervised learning, effective methods to mitigate these threats for backdoor attacks specifically designed for SSL are still lacking. The primary reason is that those backdoor defense methods designed for supervised learning heavily rely on the quantity of labeled data. However, in SSL, the extremely limited supervised information makes them ineffective or hard to implement.



\begin{figure}[t]
\centering
\includegraphics[width=0.95\columnwidth]{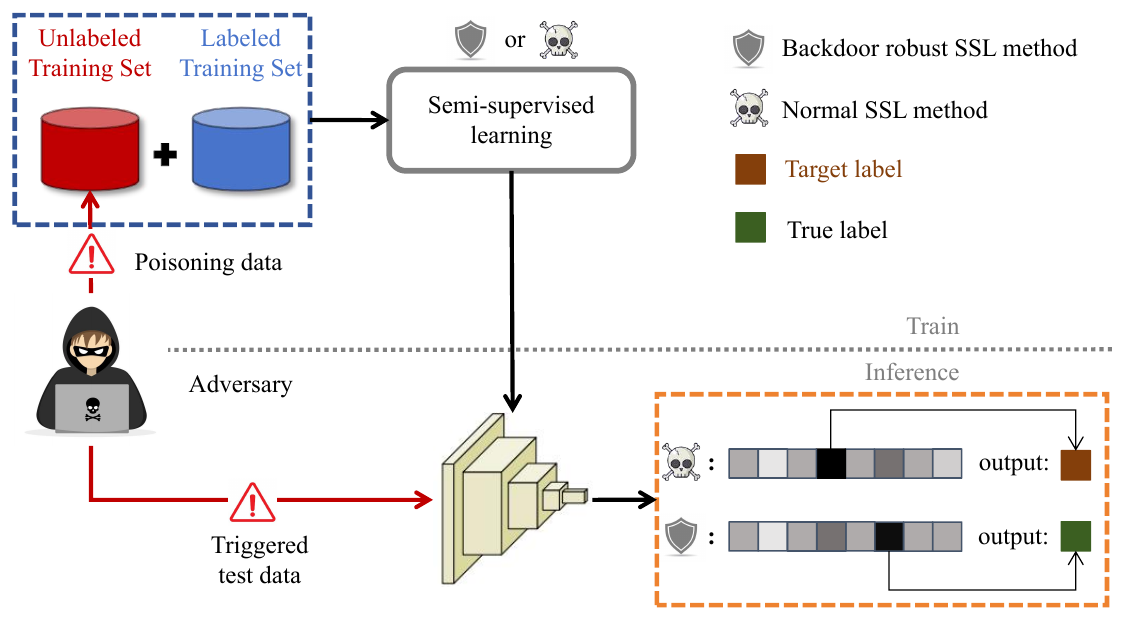}
\caption{Following previous settings\cite{yan2021deep}, poisoned data is exclusively introduced into the unlabeled set, as the labeled set is typically subjected to careful inspection. Our goal is to prevent adversaries from manipulating test data outputs from the true label to the targeted one under the poisoned dataset.}
\vspace{-5pt}
\label{framework}
\end{figure}

\begin{figure*}[t]
\centering
\includegraphics[width=0.87\textwidth]{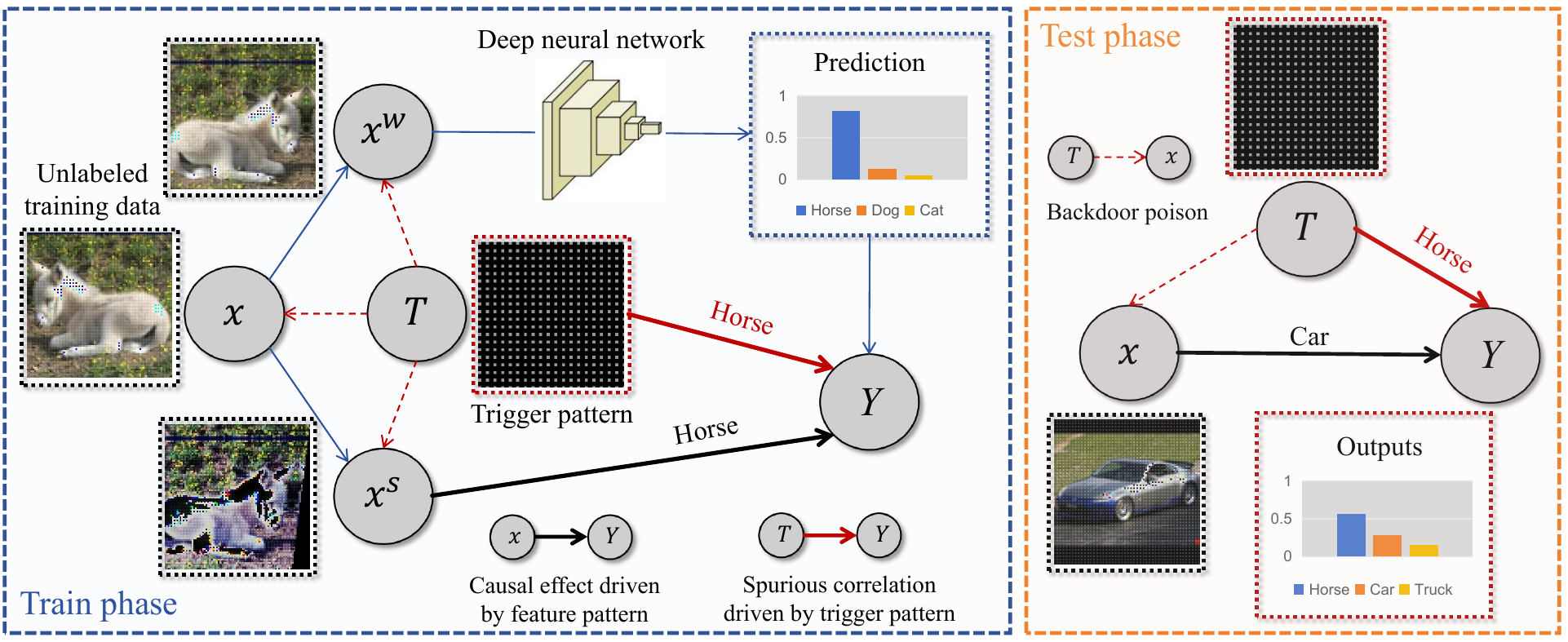}
\caption{Visualization of the mechanism behind successful backdoor attacks in SSL from a casual perspective.}
\label{causal_explanation}
\end{figure*}



Before proposing a backdoor-resistant SSL method, it is essential to understand the rationale behind the susceptibility of existing SSL methods to backdoor attacks. Deep neural networks (DNNs) are prone to learning coincidental feature associations formed between a subset of the input and target labels, which may be caused by factors such as data selection bias. These associations are referred to as spurious correlations\cite{asgari2022masktune}. In backdoor attacks, particularly clean-label variants, adversaries exploit this tendency by poisoning a small portion of training data, deliberately introducing spurious correlations between backdoor triggers and target labels while preserving the original labels.As shown in Figure \ref{causal_explanation}, SSL models capture two types of relationships: (1) legitimate causal effects between unlabeled data and their corresponding labels based on genuine feature patterns, and (2) artificial spurious correlations created by trigger patterns through pseudo-labeling. However, DNN models inherently favor learning simple, discriminative feature-category mappings\cite{zhang2023backdoor}, making them particularly susceptible to these injected spurious correlations especially during early training stages. Consequently, when these spurious correlations gradually overshadow the genuine causal relationships in the test phase, misclassifying attacked data as the target class becomes inevitable. To combat this vulnerability, we present a defense framework that addresses backdoor attacks from three distinct perspectives.


From a data perspective, we first examine the characteristics of backdoor attacks against SSL \cite{shejwalkar2023perils, wang2022invisible}. Previous studies have revealed that successful backdoor triggers often resemble constant repetitive patterns similar to high-frequency noise signals, and such patterns should ideally span the entire image space to resist the frequently used data augmentation. This insight has motivated the adoption of a Gaussian Filter as a countermeasure for implanted backdoor triggers. It effectively smooths images by convolving them with a Gaussian function, thereby attenuating the impact of these noise-like trigger patterns during training while preserving the integrity of the original image structures \cite{gwosdek2012theoretical}.

From a label perspective, we aim to mitigate the correlation between the backdoor trigger and target class from a causal perspective. As mentioned previously, when the spurious correlation driven by the implanted trigger pattern overwhelms the causal effect driven by the feature pattern, adversarial manipulation becomes inevitable. To avoid this, we innovatively replace the simplistic one-to-one relationship between backdoor trigger and target class with a more complex one-to-all relationship. We argue that building a correlation to one specific label might be easy but excluding all other categories presents a substantial challenge. Therefore, we combine consistency regularization with complementary learning to substitute the supervised learning scheme\cite{ijcai2021p423}. It encourages models to identify which categories input data does not belong to, rather than predicting the category it does belong to. 

At last, to further dilute the the influence of backdoors during the training. We broaden the correlation between the backdoor trigger and corresponding target class to all categories. It's implemented through a simple mix-up strategy. As correlations with all classes effectively negate any specific correlation with a single class, this strategy serves as a mild way to supply the disruption of backdoors. By combining all these strategies, we significantly strengthens SSL model's resilience against backdoor attacks without sacrificing its clean data accuracy. Here, we summarize our main contributions as follows:
\begin{itemize}
    \item We conduct a detailed analysis of the rationale backdoor attacks for SSL and propose the first plug-in method for SSL that can counter these attacks. 
    \item We provide a theoretical guarantee on the proposed complementary learning term to ensure that the classifier learned with complementary labels converges to the optimal one trained by traditional consistency loss.
    \item We evaluate our proposed method against a range of state-of-the-art backdoor attacks to confirm its backdoor robustness and performance on clean data.
\end{itemize}

\section{Backgroud}

In this paper, we basically follow the settings in \cite{shejwalkar2023perils} and concentrate on the backdoor attack and defense for SSL-based image classification systems. We begin by formalizing some notations, followed by the definition of adversary's objectives, capabilities and knowledge assumptions.

\noindent{\textbf{Problem Formulation:}}
In SSL, the training set is composed of both labeled and unlabeled data. Let $\mathcal{D}_l = {(x^i_l, y^i_l) : i \in [n]}$ represent the labeled dataset and $\mathcal{D}_u = {x^i_u : i \in [m]}$ denote the unlabeled dataset, where $n$ and $m$ are the quantities of labeled and unlabeled data. We follow the attack settings in previous works\cite{yan2021dehib} which assume a set of backdoor examples has been pre-generated by the attacker and successfully injected into the training dataset. Specifically, within the unlabeled set $\mathcal{D}_u$, there exists both a clean subset $\mathcal{D}_u^{cl} = {x^i_u : i \in [m^{cl}]}$ and a poisoned subset $\mathcal{D}_u^p = {x^i_u : i \in [m^p]}$, satisfying $m^{cl} + m^p = m$. In each training iteration, we sample batches $\mathcal{B}_l$ and $\mathcal{B}_u$ from the labeled dataset $\mathcal{D}_l$ and the unlabeled dataset $\mathcal{D}_u$, respectively, to serve as the training data. In the following sections, we define the classifier $f$ as: $\hat{y}=f(x)=\arg\max_{i\in[c]} g_i(x)$, where $\mathbf{g}:\mathcal{X} \to \mathbb{R}^c$ and $g_i(x)$ is the estimate of $P(y=i|x)$. Additionally, we denote $\pi_{k}=P(y=k)$ and $\bar{\pi}_{k}=P(y\neq k)$ as the prior of data belong and not belong to class $k$.

\noindent\textbf{Adversary's and Defender's Objectives:} The objective of a backdoor adversary is to install a backdoor function into the victim’s model. For an input image $x$ with the true label $ y^{\ast} $, the adversary's goal is to let the backdoored model output an desired target label $ y^t$ when the input $x$ is modified with a pre-specified backdoor trigger $T$, denoted as $x^t$. 

While defender's goal is to train a backdoor free classifier $f$ that output $f(x^t)=y^*$ using the aforementioned datasets $\mathcal{D}_l$ and $\mathcal{D}_u$, aiming for performance comparable to models trained on entirely clean data. 

\noindent\textbf{Adversary's Knowledge and Capabilities:} As discussed in Section \ref{intro}, we consider a situation where adversaries have precise knowledge of the classification task and access to the unlabeled training data. However, they only poison the unlabeled data used in the SSL pipeline, without having access to the labeled dataset or model itself. We make this assumption by presuming that, in SSL, the scarce labeled data is typically under careful selection and inspection.

\section{Method}
\subsection{Trigger Filtering}

As demonstrated and proved by the previous literature \cite{shejwalkar2023perils}, unlike attacks for supervised learning, successful backdoor attack triggers in SSL should adhere to several key principles: (1) Backdoor attacks should employ a clean-label style: for poisoned data $(x^p, y^*)$, $y^*=y^t$. (2) The backdoor trigger should span the entire image: the size of trigger $T$ should be similar to the size of input image $x$, e.g. $H \times W$ where $H$ and $W$ represent the height and width of the image (3) The backdoor trigger should be resistant to noise and its pattern should be repetitive: $f(\omega(x^t))=f(\Omega(x^t))$ where $\omega(\cdot)$ and $\Omega(\cdot)$ respectively denote weak and strong data augmentations. In addition to the backdoor attack strategies outlined by Shejwalkar et al, we find certain attacks designed for self-supervised learning, as detailed by \cite{wang2022invisible}, also prove effective in SSL contexts. 

In this section, we visualize two most successful backdoor triggers in SSL, labeled as 'Mosaic' and 'Freq'. As depicted in Figure \ref{gaussian}, the characteristics of these backdoor triggers closely resemble certain high-frequency noises (such as salt and pepper noise or line drop) in image processing that display sudden changes in local pixel values. Compared to the Mosaic trigger, the Freq trigger is less visible, especially in highlight background images, as detailed in Figure \ref{poison_vis_all} Appendix\ref{setup}). Traditional backdoor triggers, such as small white squares\cite{gu2017identifying}, pasted image parts\cite{saha2020hidden}, or adversarial patterns\cite{yan2021dehib}, can be easily filtered out by various data augmentation methods widely employed in SSL. In contrast, these backdoor attacks\cite{shejwalkar2023perils, wang2022invisible} designed for SSL are more noise-resistant and harder to detect and filter out in both training and testing phase. 

To address this issue, we propose adding a Gaussian Filter into the image pre-processing stage. As shown in Figure \ref{gaussian}, it successfully purifies the backdoor trigger pattern in the poisoned data without influencing the original data pattern by convolving local pixels with a Gaussian function $G_{i,j} = \frac{1}{2\pi \gamma ^ 2}exp({-\frac{(i - 3\gamma)^2 + (j - 3\gamma)^2}{2 \gamma ^ 2}})$, where $i,j$ are the 2D coordinate of the image and $\gamma$ is the hyper-parameter that determines both the standard deviation of Gaussian function and its kernel radius. 

\begin{figure}[t]
\centering
\includegraphics[width=0.9\columnwidth]{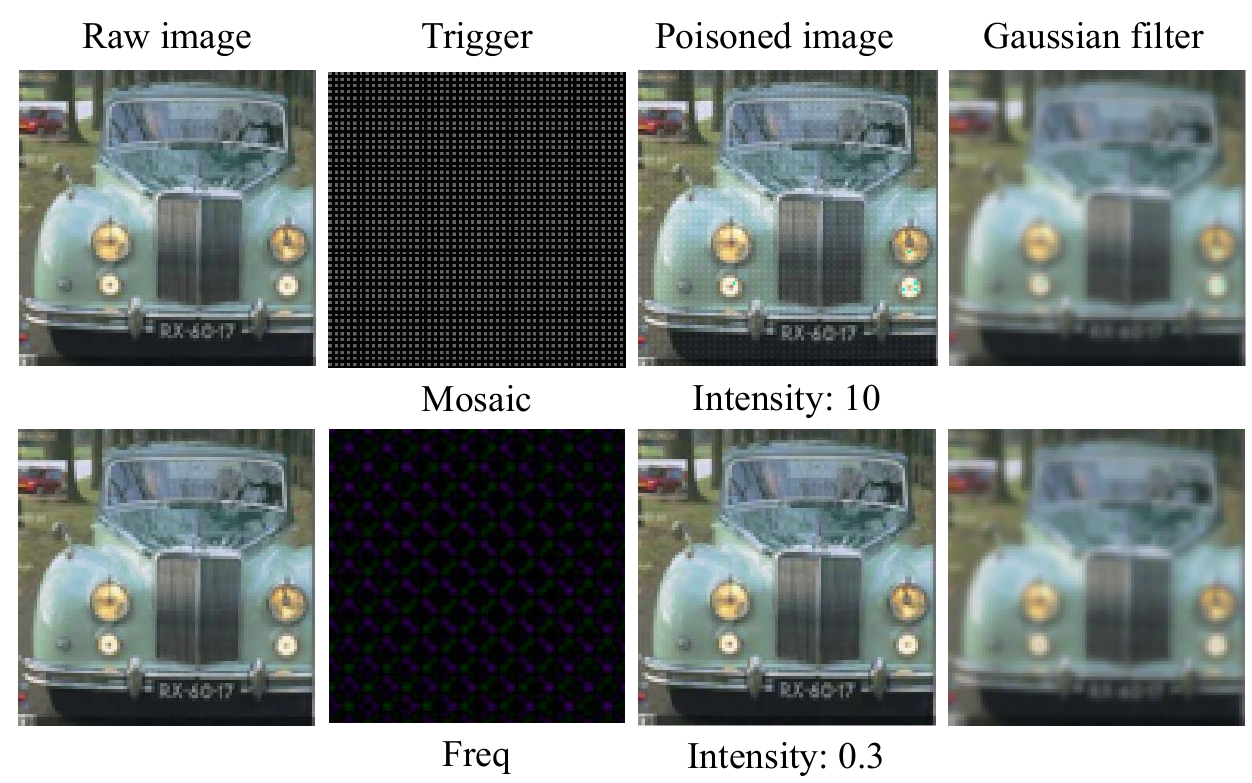}
\caption{Visualization of two successful backdoor triggers (including Gaussian Filter) under different attack intensity \cite{shejwalkar2023perils, wang2022invisible}. For enhanced visualization, the trigger patterns in the second row are displayed with a $10\times$ intensity amplification.}
\label{gaussian}
\end{figure}

\subsection{Backdoor Obstruction}
In addition to data-related perspectives, we aim to prevent the formation of spurious correlations between the trigger and the target label. As shown in Figure \ref{causal_explanation}, backdoors in SSL are introduced in a way akin to supervised learning. The adversaries presuppose that the trigger pattern in poisoned data offers a more straightforward route to the target label compared to natural feature patterns. In other words, when the model can easily capture the relationship between the feature pattern and the target class, the artificial linkage between backdoor triggers and the target label can be substantially weakened. As highlighted by \citet{shejwalkar2023perils}, the success rate of backdoor attacks in SSL tends to increase sharply within the first 5000 iterations. We also reveal that there might exist potential contentions between the learning of trigger patterns and natural feature patterns, especially in the early training stage, as detailed in Appendix \ref{contention} (Figure \ref{hist}). These insights and observations all underscore the critical importance of the early stages of training in both the backdoor implantation and its defense. The key to prevent the network from modeling spurious correlations between backdoor triggers and target classes is to substitute consistency loss, especially in the early stages of training. However, consistency loss, denoted as Eq.~\ref{consistency} plays a critical role in label propagation, making it irreplaceable in SSL.
\begin{equation} \label{consistency}
    \mathcal{L}_{con} = \sum ^ {|\mathcal{B}_u|} _{i=1}  \mathbb{I}(g_{\hat{y}}(\omega(x_u^i))\ge \tau)\ell(g(\Omega(x_u^i)), f(x))
\end{equation}

Fortunately, insights from some studies in learning from complementary labels \cite{ishida2017learning, xu2020generative, feng2020learning, gao2023unbiased} suggest an alternative approach that both facilitates label propagation and obstructs the direct correlation between the trigger and the target label: complementary learning encourage models to focus on identifying which classes the data \textit{does not} belong to, rather than focusing solely on what it \textit{does} belong to. Our intuition behind is also straightforward: although building a spurious correlation between trigger and a specific target class is simple, establishing multiple correlations to exclude all other categories presents a considerable challenge. 

Following the techniques used by \cite{yu2018learning,gao2024complementary}, we replace the consistency loss term $\mathcal{L}_{\text{con}}$ with the complementary loss $\mathcal{L}_{\text{com}}$ in Eq.~\ref{complementary}, where $\bar{y}_u$ is the estimated complementary label (denoting the classes that data \textit{does not} belong to) from $\omega(x_u)$ and $\bar{\ell}(f(x), \bar{y}_u) = \ell(\mathbf{Q}^\top g(x), \bar{y}_u)$ is the modified loss function for complementary learning. Here, $\mathbf{Q}$ represents the transition matrix that converts the predicted probability $P(y=i|x)$ to $P(y \neq j|x)$ according to the formula $P(\bar{y}=j|x) = \sum_{i \neq j} P(\bar{y}=j|y=i) P(y=i|x)$ which is derived from the definition of conditional probability. We summarize all the conditional probabilities between different classes as $Q_{ij}=P(y \neq j|y=i)$ into a transition matrix $\mathbf{Q}\in \mathbb{R}^{c \times c}$ and $Q_{ij}$ denotes the entry value in the $i$-th row and $j$-th column of transition matrix $\mathbf{Q}$.
\begin{equation} \label{complementary}
    \mathcal{L}_{com} = \sum^{|\mathcal{B}_u|}_{i=1} \bar{\ell}(g(\Omega(x_u^i)), \bar{y}_u)
\end{equation}

Inspired by the pseudo labeling strategy used in \cite{li2024instant}, we also generate the pseudo complementary labels based on the model predictions $g(x)$ and the learning effect as indicated by the number of data instances whose predictions on weakly augmented data align with those on strongly augmented data. The complementary label $\bar{y}_j^i = 1$ is generated (sampled) with the probability $(1-g_j(x^i))\cdot\sigma_t$ where $\sigma_t = \frac{1}{m} \sum_{k=1}^m \mathbb{I}(f(\omega(x^k)) = f(\Omega(x^k)))$ is the alignment ratio of current model. These approaches ensure that we adopt a conservative pseudo-labeling strategy in the early stages when the model has not yet acquired sufficient knowledge. Additionally, we take a moving average strategy to estimate the transition matrix as $\mathbf{Q}_t = \frac{1}{t} \hat{\mathbf{Q}} + \frac{t-1}{t} \mathbf{Q}_{t-1}$, where $\hat{\mathbf{Q}}$ is the estimated transition matrix by averaging the conditional probabilities $P(\bar{y}=j|x, y=i)$ on the current available batch of data $x$ in class $i$. Due to the limited space, we provide detailed pytorch-like algorithm description of complementary label generation and transition matrix estimation in Algorithm 1 and Algorithm 2 (Appendix \ref{appendix: alg}).

\subsection{Backdoor Dilution}
During the experiments, we observed that the backdoor filtering and obstruction strategies effectively defend against existing backdoor attacks, reducing the attack success rate from 90\% to 1\%. However, these strategies also reduce the model's accuracy on clean data, for reasons that are straightforward. The Gaussian Filter used in the backdoor filtering process tends to blur the input images, while the complementary learning approach used in backdoor obstruction requires more training iterations to achieve results comparable to those of normal supervised learning. To improve the model's performance without increasing the risk of attacks, we implemented a simple data mix-up strategy on the unlabeled data and their candidate labels. As demonstrated by Figure \ref{mix}, data mix-up does not compromise the trigger pattern in the poisoned image. We intentionally associated the trigger pattern with the label of a mixed class (horse), in addition to the original target class (bird). By distributing such trigger patterns across images of all classes during the training stage, we can effectively neutralize that specific association between the backdoor trigger and a single target class, thereby weakening the spurious correlation that a backdoor trigger could otherwise establish. 

Specifically, we divide the training process into two stages. In the first stage, we employ a supervised loss on the labeled data and a complementary loss on the unlabeled data, as described in Eq.\ref{loss1}. This approach ensures that the model focuses on capturing the feature patterns rather than the trigger patterns during the initial training phase.
\begin{equation} \label{loss1}
    \mathcal{L}=\sum^{|\mathcal{B}_l|}_{i=1} {\ell}(g(\omega(x_l^i)), {y}_l^i) + \sum^{|\mathcal{B}_u|}_{i=1} \bar{\ell}(g(\Omega(x_u^i)), \bar{y}_u)
\end{equation}

\vspace{-4pt}
In the second stage, we implement a data mix-up between the unlabeled data predicted with high confidence and the labeled data. Unlike traditional mix-up techniques that sample the mixing coefficient $\lambda$ from a Beta distribution \cite{zhang2017mixup}, we let the proportion of clean labeled data is greater than that of potentially poisoned unlabeled data. It ensures that the trigger pattern becomes more associated with the mixed class rather than the target class. We achieve this by defining $\lambda^{'}$ as $\lambda^{'}=\max(\lambda,1-\lambda)$.
\begin{equation} \label{eq_mix}
\Tilde{x}^j = \lambda^{'} x^i_l + (1-\lambda^{'})x^j_u, \quad \Tilde{y}^j = \lambda^{'} y^i_l + (1-\lambda^{'})f(x^j_u)
\end{equation}
We employ a combination of loss on the mixed data and consistency loss as the loss function, detailed in Eq.\ref{loss2}, where $\mathbb{T}(\cdot)$ is the threshold function that determines which unlabeled data are included in the training:
\begin{equation} \label{loss2}
\begin{aligned}
\mathcal{L} = & \sum^{|\mathcal{B}_l|}_{i=1} \ell(g(\omega(x_l^i)), y_l^i) + \alpha \sum^{|\mathcal{B}_u|}_{i=1} \mathbb{T}(x^i_u) \ell(g(\Tilde{x}^i), \Tilde{y}^i) \\
& + (1-\alpha) \sum^{|\mathcal{B}_u|}_{i=1} \mathbb{T}(x^i_u) \ell(g(\Omega(x_u^i)), f(x^i_u))
\end{aligned}
\end{equation}
For a complete training procedure, please refer to the Algorithm 3 in Appendix \ref{appendix: alg}.

\begin{figure}[t]
\centering
\includegraphics[width=0.98\columnwidth]{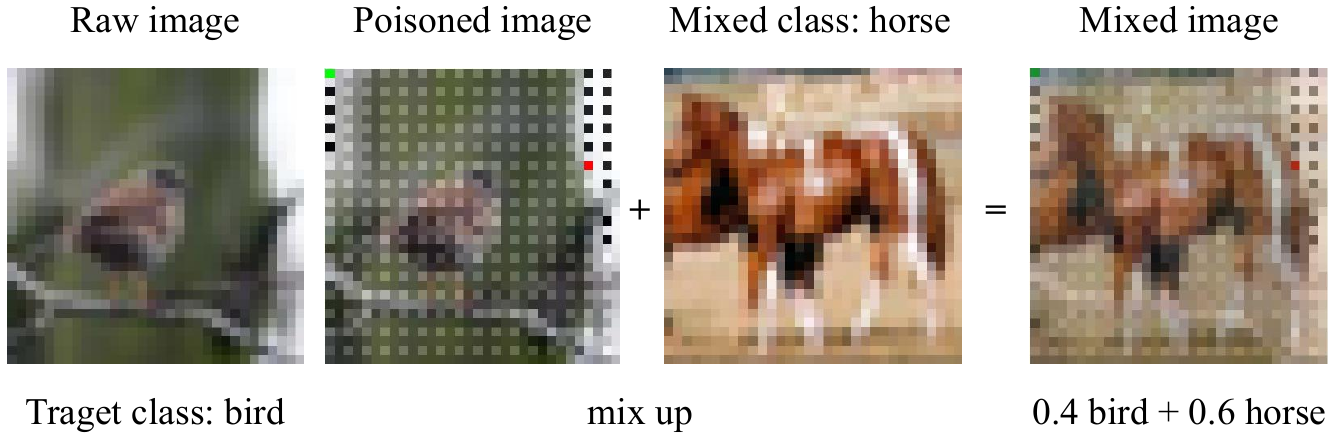}
\caption{The data mix-up does not compromise the trigger pattern, such that the trigger pattern becomes more associated with the class "horse" rather than target class "bird". Similar phenomenon also exists in many other backdoor attack triggers.}
\label{mix}
\end{figure}

\section{Theoretical Analysis}
After proposing the backdoor defense strategy, we provide a theoretical analysis of substituting the traditional consistency loss with our proposed complementary loss in the aspect of generalization. We demonstrate that, under reasonable assumptions, optimizing this new loss term in Eq.\ref{loss1} can achieve the same optimal classifier as would be obtained by minimizing the original consistency loss in \cite{sohn2020fixmatch}. Moreover, we further provide an upper bound for the estimation error of our method. Before presenting the main results, we first define the true risk associated with the classification model as $R(f)=\mathbb{E}_{(x,y)} \left[ \ell(f(x),y) \right]$ and the risk with respect to complementary labels as $\bar{R}(f)=\mathbb{E}_{(x,\bar{y})}\left[\bar{\ell}(f(x),\bar{y})\right]$. In the proposed method, we encourage the model to conduct complementary learning on unlabeled data. Then we define the risk on all training data as $\Tilde{R}(f)= {R}_l(f)+ {R}_u(f) = \mathbb{E}_{(x_l,y_l)} \left[ \ell(f(x_l),y_l) \right]+\mathbb{E}_{(x_u,\bar{y})}\left[\bar{\ell}(f(x_u),\bar{y})\right]$. Our objective is to learn an effective classification model by minimizing the empirical risk $\hat{R}(f)=\hat{R}_l(f)+\hat{R}_u(f)$. It is important to note that during training, since the labels of unlabeled data are inaccessible, we train the model with $\hat{R}^{'}_u(f)$ instead of $\hat{R}_{u}(f)$ using the output pseudo label $\hat{\bar{y}}_u$. We first demonstrate that the transition from consistency loss to complementary loss ensures the identity of the optimal classifier, given a reasonable assumption:
\begin{assumption} \label{identity}
    By minimizing the expected risk $R(f)$ on the training data, including both $R_l(f)$ and $R_u(f)$, the optimal mapping $\mathbf{g}^*$ satisfies $g_i^*(x) = P(y=i|x), \forall i \in [c]$.
\end{assumption}

\begin{theorem}
    Suppose that transition matrix $\mathbf{Q}$ is invertible and Assumption \ref{identity} is satisfied, the minimizer $\bar{f}^*$ of $\bar{R}(f)$ coincides with the minimizer $f^*$ of $R(f)$, i.e., $\bar{f}^*=f^*$.
\end{theorem}

For most loss functions like cross entropy, Assumption \ref{identity} can be provably satisfied\cite{yu2018learning}. Once recognizing the identifiability of the optimal classifier derived from complementary loss and consistency loss, we further provide generalization analysis on our proposed method which implies that the classifier $\hat{f}^{'}$ derived by the proposed method converges to the optimal classifier $f^*$.

\vspace{-5.5pt}
\begin{theorem} \label{bound}
    Suppose $\bar{\pi}_k$ and $\pi_k$ are given. Let the loss function $\ell(\cdot)$ on labeled and loss function $\bar{\ell}(\cdot)$ on unlabeled data be upper bounded respectively by $M_1$ and $M_2$. For some $\epsilon>0$, if $\sum^m_{i=1}\sum^c_{k=1}|\hat{y}^{ik}_u-y^{ik}_u|/m\leq \epsilon$. Then, for any $\delta>0$, with the probability $1-c\delta$:
    \begin{equation}
    \begin{aligned}
        &\Tilde{R}(\hat{f}^{'})- \Tilde{R}(f^*) \leq \sum_{k=1}^c \Bigg( 4c \pi_k \mathfrak{R}_{n_{k}}(\mathcal{H}) + 4c \bar{\pi}_k \mathfrak{R}_{m_{k}}(\mathcal{H})  \\
        &+ 2\pi_k M_1\sqrt{\frac{\log{1/\delta}}{2n_{k}}} + 2\bar{\pi}_k M_2\sqrt{\frac{\log{1/\delta}}{2m_{k}}} \Bigg) + 2M_2\epsilon,
    \end{aligned}
    \end{equation}
    where $y^{i}_u$ represents the true label of unlabeled data $x_u^i$ and $\hat{y}^{i}_u$ is the estimated pseudo label; $n_k$ represents the the numbers of labeled data whose labels are $y=k$ and $m_k$ represents the the numbers of unlabeled data whose complementary labels are $\bar{y}=k$. $\mathfrak{R}_{n}(\mathcal{H})=\mathbb{E}\left[ \sup_{h_k(x) }\frac{1}{n}\sum_{j=1}^{n}\sigma_j h_k(x)\right]$is the Rademacher complexity and $\{\sigma_1,\cdots,\sigma_n\}$ are Rademacher variables uniformly distributed from $\{-1,1\}$.
\end{theorem}
\vspace{-5pt}

Theorem \ref{bound} illustrates the generalization error bound of our proposed method using Eq.\ref{loss1} as the loss function. As $m_k$ and $n_k$ approach infinity and $\epsilon$ approaches zero, the empirical risk minimizer trained converges to the true risk minimizer with a high probability. These results provide a theoretical guarantee for substituting the consistency loss with the complementary loss in the first stage of training. We leave the detailed proof in the Appendix \ref{appdendix: proof}.

\begin{table*}[htb] 
\caption{The attack success rate (\textit{ASR}\%)  and the clean accuracy (\textit{CA}\%) our proposed BI against 5 representative backdoor attacks (More results on other SSL methods are provided in Table \ref{tab:more_results}, Appendix \ref{appendix: other ssl methods}). Best \textit{ASR} \% is highlighted in \textbf{bold}. } \label{tab:main_results}
\centering
\begin{subtable}{\textwidth}
\centering
\small
\begin{tabular} {|c|c|c|c|c|c|c|c|c|c|c|c|c|c|c|}
  \hline
  \multirow{6}{*}{\rotatebox{90}{\textbf{CIFAR10}}} & \multirow{2}{*}{Algorithm}  & \multicolumn{2}{c|}{CL-Badnets} & \multicolumn{2}{c|}{Narcissus} & \multicolumn{2}{c|}{DeHiB*} & \multicolumn{2}{c|}{Mosaic} & \multicolumn{2}{c|}{Freq}\\ \cline{3-12}
  
  && CA $\uparrow$ & ASR $\downarrow$ & CA $\uparrow$ & ASR $\downarrow$ & CA $\uparrow$ & ASR $\downarrow$ & CA $\uparrow$ & ASR $\downarrow$ & CA $\uparrow$ & ASR $\downarrow$\\ \cline{2-12}
  &Fixmatch & 93.9 & 13.4 & 94.2 & 1.3 & 94.0 & 35.8 & {94.2} & 93.8 & 94.8 & 90.2 \\ \cline{2-12}
  &Flexmatch & {94.2} & 12.4 & {94.9} & 1.1 & 94.2 & 16.9 & {94.3} & 90.1 & 95.0 & 93.4 \\ \cline{2-12}
  &Fixmatch w/ BI & 93.4 & $\textbf{1.4}_{-12.0}$ & 93.5 & $\textbf{0.0}_{-1.3}$ & 92.9 & $\textbf{0.1}_{-35.7}$ & 93.4 & $\textbf{2.5}_{-91.3}$ & 93.8 & ${0.7}_{-89.5}$ \\ \cline{2-12}
  &Flexmatch w/ BI & 92.5 & $2.5_{-10.9}$ & 93.1 & $\textbf{0.0}_{-1.1}$ & 93.4 & $1.1_{-15.8}$ & 93.0 & ${4.1}_{-86.0}$ & 93.0 & $\textbf{0.4}_{-93.0}$ \\\hline
\end{tabular}
\end{subtable}

\begin{subtable}{\textwidth}
\centering
\small
\begin{tabular} {|c|c|c|c|c|c|c|c|c|c|c|c|c|c|c|}
  \hline
  \multirow{6}{*}{\rotatebox{90}{\textbf{SVHN}}} & \multirow{2}{*}{Algorithm}  & \multicolumn{2}{c|}{CL-Badnets} & \multicolumn{2}{c|}{Narcissus} & \multicolumn{2}{c|}{DeHiB*} & \multicolumn{2}{c|}{Mosaic} & \multicolumn{2}{c|}{Freq}\\ \cline{3-12}
  
  && CA $\uparrow$ & ASR $\downarrow$ & CA $\uparrow$ & ASR $\downarrow$ & CA $\uparrow$ & ASR $\downarrow$ & CA $\uparrow$ & ASR $\downarrow$ & CA $\uparrow$ & ASR $\downarrow$\\ \cline{2-12}
  &Fixmatch  & {94.9} & 3.1 & 94.2 & 0.0 & 94.8 & 3.2 & 94.5 & 97.1 & 93.8 & 84.6 \\ \cline{2-12}
  &Flexmatch  & 88.9 & 1.2 & 86.1 & 0.0 & 86.8 & 2.2 & 83.9 & 50.1 & 94.9 & 86.4 \\ \cline{2-12}
  &Fixmatch w/ BI & 94.4 & $0.2_{-2.90}$ & {94.6} & $0.0_{-0.0}$ & {94.9} & $\textbf{0.4}_{-2.80}$ & {95.1} & ${0.5}_{-96.6}$ & 94.7 & $\textbf{1.2}_{-83.4}$ \\ \cline{2-12}
  &Flexmatch w/ BI & 93.9 & $\textbf{0.0}_{-1.20}$ & 94.1 & $0.0_{-0.0}$ & 94.4 & $0.5_{-1.70}$ & 94.3 & $\textbf{0.3}_{-49.8}$ & {95.0} & ${1.4}_{-85.0}$ \\\hline
\end{tabular}
\end{subtable}

\begin{subtable}{\textwidth}
\centering
\small
\begin{tabular} {|c|c|c|c|c|c|c|c|c|c|c|c|c|c|c||c|c|}
  \hline
  \multirow{6}{*}{\rotatebox{90}{\textbf{STL10}}} & \multirow{2}{*}{Algorithm} & \multicolumn{2}{c|}{CL-Badnets} & \multicolumn{2}{c|}{Narcissus} & \multicolumn{2}{c|}{DeHiB} & \multicolumn{2}{c|}{Mosaic} & \multicolumn{2}{c|}{Freq}\\ \cline{3-12}
  
  && CA $\uparrow$ & ASR $\downarrow$ & CA $\uparrow$ & ASR $\downarrow$ & CA $\uparrow$ & ASR $\downarrow$ & CA $\uparrow$ & ASR $\downarrow$ & CA $\uparrow$ & ASR $\downarrow$\\ \cline{2-12}
  
  &Fixmatch  & {92.2} & 13.1 & {92.1} & \textbf{0.0} & {92.0} & 2.2 & 91.8 & 92.4 & 91.7 & 91.5\\ \cline{2-12}
  &Flexmatch & 88.1 & 6.5 & 88.4 & 0.9 & 87.8 & 1.7 & 87.8 & 49.8 & 90.9 & 75.8 \\ \cline{2-12}
  &Fixmatch w/ BI & 91.7 & $\textbf{0.0}_{-13.1}$ & 91.9 & $0.1_{+0.1}$ & 91.7 & $\textbf{0.4}_{-1.80}$ & 92.3 & $\textbf{1.2}_{-91.2}$ & 92.4 & $\textbf{0.2}_{-91.3}$ \\ \cline{2-12}
  &Flexmatch w/ BI & 91.4 & $3.4_{-3.10}$ & {92.1} & $0.1_{-0.8}$ & 91.4 & $0.5_{-1.20}$ & {93.1} & ${2.3}_{-47.5}$ & {92.8} & ${0.3}_{-75.5}$ \\\hline
\end{tabular}
\end{subtable}

\begin{subtable}{\textwidth}
\centering
\small
\begin{tabular} {|c|c|c|c|c|c|c|c|c|c|c|c|c|c|c|}
  \hline
  \multirow{6}{*}{\rotatebox{90}{\textbf{CIFAR100}}} & \multirow{2}{*}{Algorithm} & \multicolumn{2}{c|}{CL-Badnets} & \multicolumn{2}{c|}{Narcissus} & \multicolumn{2}{c|}{DeHiB*} & \multicolumn{2}{c|}{Mosaic} & \multicolumn{2}{c|}{Freq}\\ \cline{3-12}
  && CA $\uparrow$ & ASR $\downarrow$ & CA $\uparrow$ & ASR $\downarrow$ & CA $\uparrow$ & ASR $\downarrow$ & CA $\uparrow$ & ASR $\downarrow$ & CA $\uparrow$ & ASR $\downarrow$\\ \cline{2-12}
  &Fixmatch & 70.6 & 22.0 & 71.4 & 1.1 & 71.4 & 14.5 & 71.1 & 91.8 & 70.8 & 90.3 \\ \cline{2-12}
  &Flexmatch & 71.9 & 23.4 & {72.4} & 5.9 & {71.8} & 6.8 & {72.5} & 94.6 & {71.9} & 76.4  \\ \cline{2-12}
  &Fixmatch w/ BI & 70.6 & $\textbf{0.7}_{-21.3}$ & 70.9 & $1.5_{+0.4}$ & 71.0 & $\textbf{1.4}_{-13.1}$ & 65.4 & $\textbf{3.2}_{-88.6}$ & 67.6 & $\textbf{0.4}_{-89.9}$ \\ \cline{2-12}
  &Flexmatch w/ BI & {72.0} & $0.9_{-22.5}$ & 71.1 & $\textbf{0.1}_{-5.8}$ & 71.5 & ${2.5}_{-4.30}$ & 66.3 & ${5.7}_{-88.9}$ & 68.9 & ${0.7}_{-75.7}$ \\\hline
\end{tabular}
\end{subtable}
\end{table*}

\section{Experiment}

\begin{table*}[htb] 
\caption{Comparison between SOTA learning-algorithm-agnostic defenses and our proposed BI based on Fixmatch against two selected effective backdoor attacks (Mosaic and Freq). Best \textit{CA}\% and \textit{ASR}\% (excluding No defense) are highlighted in \textbf{bold}.} \label{tab:defense_results}
    \centering
    \begin{subtable}{\textwidth}
    \centering
    \small
    \begin{tabular} {|c|c|c|c|c|c|c|c|c|c|c|c|c|c|}
      \hline
      \multirow{6}{*}{\rotatebox{90}{\textbf{MOSAIC}}} & \multirow{2}{*}{Dataset} & \multicolumn{2}{c|}{No defense} & \multicolumn{2}{c|}{FT} & \multicolumn{2}{c|}{FP} & \multicolumn{2}{c|}{NAD} & \multicolumn{2}{c|}{ABL} & \multicolumn{2}{c|}{BI} \\ \cline{3-14}
      && CA $\uparrow$ & ASR $\downarrow$ & CA $\uparrow$ & ASR $\downarrow$ & CA $\uparrow$ & ASR $\downarrow$ & CA $\uparrow$  &ASR $\downarrow$ & CA $\uparrow$ & ASR $\downarrow$ & CA $\uparrow$ & ASR $\downarrow$ \\ \cline{2-14}
      &CIFAR10 & 94.2 & 93.8 & 90.7 & 86.5 & 91.5 & 80.6 & 86.5 & 59.8 & \textbf{94.4} & 92.6 & 93.4 & \textbf{2.5} \\ \cline{2-14}
      &SVHN & 94.5 & 97.1 & 93.4 & 95.2 & 95.1 & 98.1 & 82.3 & 92.1 & 94.0 & 97.1 & \textbf{95.1} & \textbf{0.5} \\ \cline{2-14}
      &STL10 & 91.8 & 92.4 & 86.7 & 90.3 & 87.8 & 84.6 & 74.5 & 91.8 & 90.9 & 89.5 & \textbf{92.0} & \textbf{1.2} \\ \cline{2-14}
      &CIFAR100 & 71.1 & 91.8 & 64.3 & 79.4 & 65.9 & 80.9 & 56.8 & 70.2 & \textbf{69.7} & 90.3 & 65.4 & \textbf{3.2} \\ \hline
    \end{tabular}
    \end{subtable}

    \begin{subtable}{\textwidth}
    \centering
    \small
    \begin{tabular} {|c|c|c|c|c|c|c|c|c|c|c|c|c|c|}
      \hline
      \multirow{6}{*}{\rotatebox{90}{\textbf{FREQ}}} & \multirow{2}{*}{Dataset} & \multicolumn{2}{c|}{No defense} & \multicolumn{2}{c|}{FT} & \multicolumn{2}{c|}{FP} & \multicolumn{2}{c|}{NAD} & \multicolumn{2}{c|}{ABL} & \multicolumn{2}{c|}{BI} \\ \cline{3-14}
      && CA $\uparrow$ & ASR $\downarrow$ & CA $\uparrow$ & ASR $\downarrow$ & CA $\uparrow$ & ASR $\downarrow$ & CA $\uparrow$  &ASR $\downarrow$ & CA $\uparrow$ & ASR $\downarrow$ & CA $\uparrow$ & ASR $\downarrow$ \\ \cline{2-14}
      &CIFAR10 & 94.8 & 90.2 & 90.4 & 77.2 & 91.5 & 79.4 & 83.0 & 44.9 & \bf {94.0} & 88.3 & 93.8 & \textbf{0.7} \\ \cline{2-14}
      &SVHN & 93.8 & 84.6 & 94.0 & 87.1 & 94.3 & 82.2 & 90.9 & 77.4 & \bf 95.1 & 92.8 & {94.7} & \textbf{1.2} \\ \cline{2-14}
      &STL10 & 91.7 & 91.5 & 86.7 & 90.3 & 87.8 & 84.6 & 74.5 & 91.8 & 90.9 & 89.5 & \bf{91.6} & \textbf{0.2} \\ \cline{2-14}
      &CIFAR100 & 71.1 & 91.8 & 64.9 & 84.2 & 62.1 & 77.3 & 59.1 & 82.3 & \textbf{68.2} & 89.4 & 67.6 & \textbf{0.4} \\ \hline
    \end{tabular}
    \end{subtable}
\end{table*}

\begin{table*}[ht] 
\caption{Ablation study on "Mosaic" attack. Best \textit{CA}\% and \textit{ASR}\% (excluding Fixmatch) are highlighted in \textbf{bold}.}
\centering
\small
{
    \resizebox{0.95\textwidth}{!}
    {
        \begin{tabular}{|c|c|c|c|c|c|c|c|c|c|c|c|}
        \hline
        \multicolumn{4}{|c|}{Dataset} & \multicolumn{2}{c|}{CIFAR10} & \multicolumn{2}{c|}{CIFAR100} & \multicolumn{2}{c|}{SVHN} & \multicolumn{2}{c|}{STL10} \\ \hline
        Fixmatch & Gaussian Filter & $\mathcal{L}_{comp}$ & trigger mixup & CA $\uparrow$ & ASR$\downarrow$ & CA $\uparrow$ & ASR$\downarrow$ & CA $\uparrow$ & ASR$\downarrow$ & CA $\uparrow$ &ASR$\downarrow$\\ \hline 
        \checkmark &&& &{94.2}&93.8&{71.1}&91.8&94.5&92.4&91.8&97.1 \\ \hline
        \checkmark &\checkmark&& &92.6&8.9&59.4&11.2&93.8&1.1&\textbf{92.6}&4.5 \\\hline
        \checkmark &&\checkmark&&86.4&1.2&53.1&\textbf{0.2}&90.5&0.4&85.5&0.6 \\\hline
        \checkmark &&&\checkmark& \textbf{93.9}&26.7&\textbf{70.4}&53.8&\textbf{95.1}&12.0&91.8&28.9 \\\hline
        \checkmark &\checkmark &&\checkmark& 93.1 & 5.7 & 65.9 & 12.4 & 94.8 & 6.5 & 91.7 & 4.3 \\\hline
        \checkmark &\checkmark&\checkmark& & 84.1 & \textbf{0.4} & 51.3 & 0.3 & 85.4 & \textbf{0.0} & 85.3 & \textbf{0.1} \\\hline
        \checkmark &\checkmark&\checkmark& \checkmark &93.4&2.5&65.4&3.2&\textbf{95.1}&0.5&{92.3}&1.2\\ \hline
        \end{tabular}
    }
}
\label{tab: ablation}
\end{table*}

\subsection{Experimental Setup}
\noindent{\textbf{Datasets and Implementations.}}
To assess the performance and efficacy of our proposed backdoor defense method, we conduct experiments on four widely recognized datasets: CIFAR10, SVHN, CIFAR100, and STL10. Following prior research\cite{shejwalkar2023perils}, we vary the amounts of labeled data and backbone models (WideResnet with different width) across different datasets: 4000 for CIFAR10, 100 for SVHN, 1000 for STL10 and 2500 for CIFAR100. To ensure a fair comparison, we adhere to the experimental setup described in \cite{shejwalkar2023perils}, which involves poisoning 0.2\% of the entire dataset while maintaining the same attack intensity. We also incorporate some of their original results for comparison. Comprehensive details on the implementation of backdoor attack triggers are provided in Appendix \ref{setup}.


\noindent{\textbf{Attack and Defense Baselines.}}
To evaluate defense effects against backdoor threats, we test five representative strategies. Specifically, we chose CL-Badnets\cite{gu2017identifying}, Narcissus\cite{chen2023practical}, DeHiB\cite{yan2021dehib}, Mosaic\cite{shejwalkar2023perils} and Freq\cite{wang2022invisible} for validation (details are left in Appendix\ref{appdendix: attack setting}). DeHiB* denotes the original results from \cite{yan2021dehib} which had access to labeled data, while DeHiB refers to the results reproduced by \cite{shejwalkar2023perils} without access to labeled data. This discrepancy arises due to the lack of detailed implementation information for DeHiB on SVHN and STL10 datasets, which led to our inability to replicate the results reported in the original study. For backdoor defense, we compare the proposed Backdoor Inhibitor (BI) with four existing methods: Fine-tuning (FT), Fine-pruning (FP)\cite{liu2018fine}, Neural Attention Distillation (NAD)\cite{li2021neural}, and Anti-Backdoor Learning (ABL)\cite{li2021anti}. Given the scarcity of specialized backdoor defense methods for SSL, we utilize different types of defense strategies from supervised learning (image classification) as baseline comparisons.

\noindent{\textbf{Evaluation Metrics.}}
In this study, we use two main performance metrics (\textit{CA \& ASR}) as follows: (1) \textit{Clean accuracy (CA)} measures the accuracy of a model on clean test data without any backdoor trigger \textit{T}. It is vital for backdoored models to maintain high \textit{CA} to ensure that the backdoor attack does not compromise their benign functionality under the attack. (2) \textit{Backdoor attack success rate (ASR)} measures the success rate of manipulating a model's output when test data from non-target classes are patched with the trigger \textit{T}. For an effective defense method, the backdoored model should achieve a low \textit{ASR} to ensure its robustness.

\subsection{Effectiveness of Backdoor Invalidator}
We evaluate the the proposed method (denoted as BI) as a plug-in backdoor defense strategy by integrating it with existing popular SSL methods.  As demonstrated in Table \ref{tab:main_results} and Table \ref{tab:more_results} (Appendix\ref{appendix: other ssl methods}), our method significantly lowers the \textit{ASR} while preserving performance on clean data. In addition to the outstanding performance, we observed several other interesting phenomena: (1) Our method achieves better performance on clean data in those high resolution image dataset like STL10 compared to CIFAR10 and CIFAR100. We hypothesize that the backdoor filtering component (Gaussian Filter) of our strategy may inadvertently remove some semantic information from the original data. (2) In most settings, Fixmatch with BI exhibits a lower ASR compared to Flexmatch with BI. We believe this is partly because Flexmatch employs a more aggressive thresholding method that makes greater use of unlabeled data, especially in the early stages of training. It is important to note that although we make certain assumptions about the characteristics of backdoor triggers based on the conclusions drawn in \cite{shejwalkar2023perils}, this does not compromise the generality of the proposed method. Our approach is designed to tackle more stealthy and targeted attacks that are specifically crafted to exploit vulnerabilities in SSL, let alone those attack methods for supervised learning. As shown in Table \ref{tab:main_results} and Table \ref{tab:more_results}, for conventional attacks like CL-Badnets\cite{gu2017identifying}, Narcissus\cite{chen2023practical} and DeHiB\cite{yan2021dehib}, BI also lowers the \textit{ASR} without sacrificing performance on clean data.

We then compared BI with existing backdoor defense methods. As shown in Table \ref{tab:defense_results}, our proposed BI is essentially the only effective defense strategy against previously successful attacks (Mosaic and Freq). However, we also acknowledge that BI sometimes compromises clean data accuracy to enhance backdoor defense effectiveness. Furthermore, we evaluate the performance of BI under varying numbers of labeled data and poisoned data, as illustrated by Table \ref{tab: label_num} and Table \ref{tab: poison_num} in Appendix \ref{appedix: label_num} and Appendix \ref{appendix: poison_num}, respectively. These results demonstrate that BI consistently achieves satisfactory backdoor defending ability (\textit{ASR}) across different quantities of labels and poison ratios. For a more comprehensive analysis, please refer to Appendix \ref{appedix: label_num} and Appendix \ref{appendix: poison_num}.
\begin{figure}[t]
    \centering
    \includegraphics[width=0.49\columnwidth]{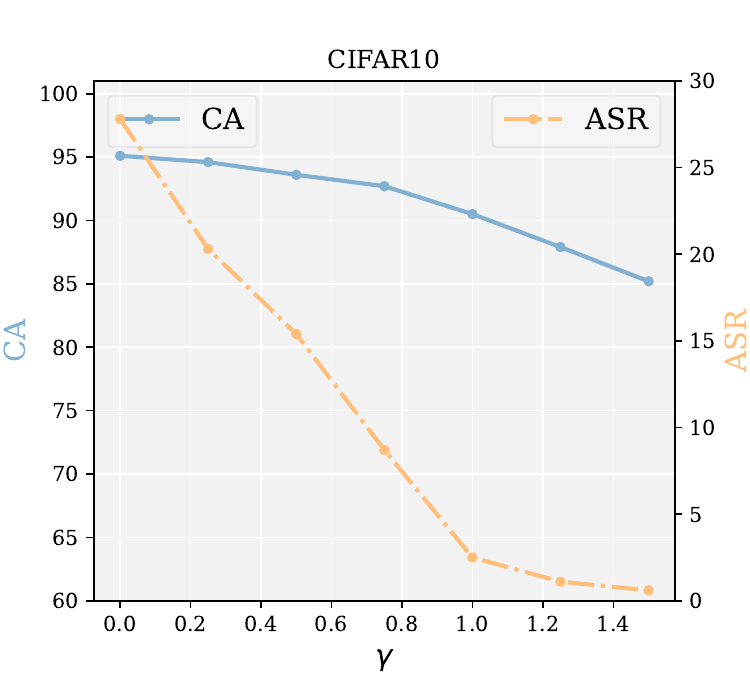}
    \includegraphics[width=0.49\columnwidth]{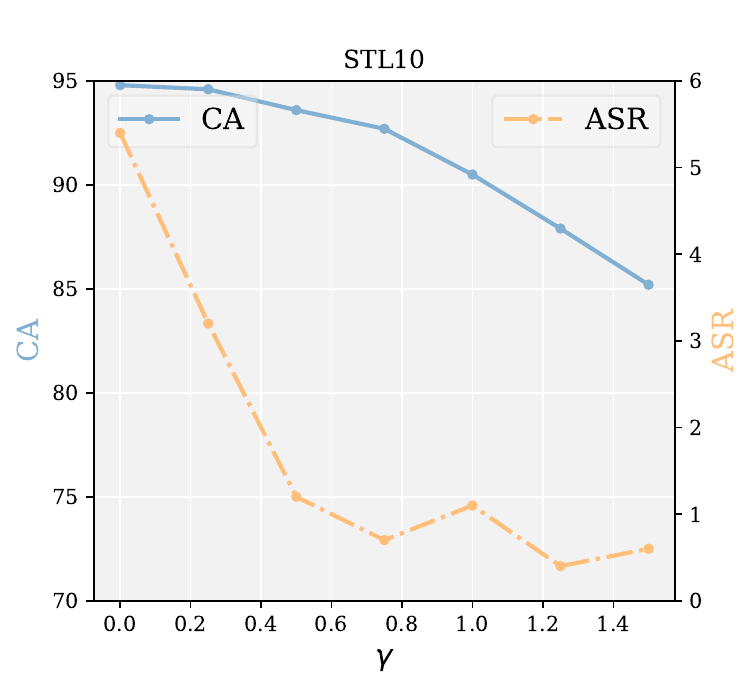}
    \caption{Sensitivity analysis on $\gamma$.}
    \label{fig: gamma}
\end{figure}
\subsection{Ablation Study}
To better understand the contributions of each component in our proposed BI method, we conduct a detailed ablation study on three main components: Gaussian Filter, complementary learning term $\mathcal{L}_{comp}$, and trigger mixup. Table \ref{tab: ablation} illustrates the impact of each component when combined with an existing SSL method. The results show that all three components can independently reduce \textit{ASR} while having different impacts on \textit{CA}. Specifically, Gaussian Filter and complementary learning term $\mathcal{L}_{comp}$ significantly decrease the backdoor risks; however, they also compromise the accuracy on clean data. This can be attributed to their partial impairment of the pseudo-labeling mechanism, which has been proven crucial in existing SSL methods \cite{li2011towards}. At the same time, trigger mixup acts as a relatively mild backdoor defense strategy and performs better in datasets with fewer total categories. This observation may be explained by the fact that, although the backdoor trigger is associated with all categories, it is linked much more frequently with the target class compared to other classes in datasets with a larger number of categories. Ultimately, combining these techniques allows them to complement each other, enhancing our goal of developing a high-performance, backdoor-robust SSL algorithm.

\subsection{Sensitivity Analysis}
In this section, we present a sensitivity analysis of the hyperparameters in our proposed method. It is crucial to highlight that in this specific experiment, the choice of the target class significantly affects the effectiveness of the backdoor attack and the identification of the optimal hyperparameters, which we give a more detailed illustration in Appendix \ref{appendix: target class}. Given the impracticality of testing every category as the target class, we will use class 0 (the first class in the list) as the target class in the subsequent discussion. As depicted in Figure \ref{fig: gamma}, the Gaussian Filter radius $\gamma$ acts as a moderator between \textit{CA} and \textit{ASR}. Higher values of $\gamma$ effectively mitigate the risk of backdoor attacks but at the expense of clean data accuracy. As shown in Table \ref{tab:main_results}, this trade-off is more noticeable in lower-resolution image datasets such as CIFAR10 and CIFAR100, where excessive blurring renders those images even unrecognizable. Regarding the trigger mixup coefficient $\alpha$ in Eq.\ref{loss2}, we employ a cosine scheduler defined by $\alpha_t = \alpha_{\min} + \frac{1}{2}(1 - \alpha_{\min})(1 + \cos(\frac{t}{t_{\max}}))$. This approach ensures the gradual integration of the consistency loss, where $t$ represents the current iteration number and $t_{\max}$ the maximum number of iterations.As illustrated in Figure \ref{fig: alpha}, results keep stable across different $\alpha_{\min}$ and achieves satisfying \textit{CA} and \textit{ASR} when $\alpha_{\min}>0.2$.

\begin{figure}[t]
    \centering
    \includegraphics[width=0.49\columnwidth]{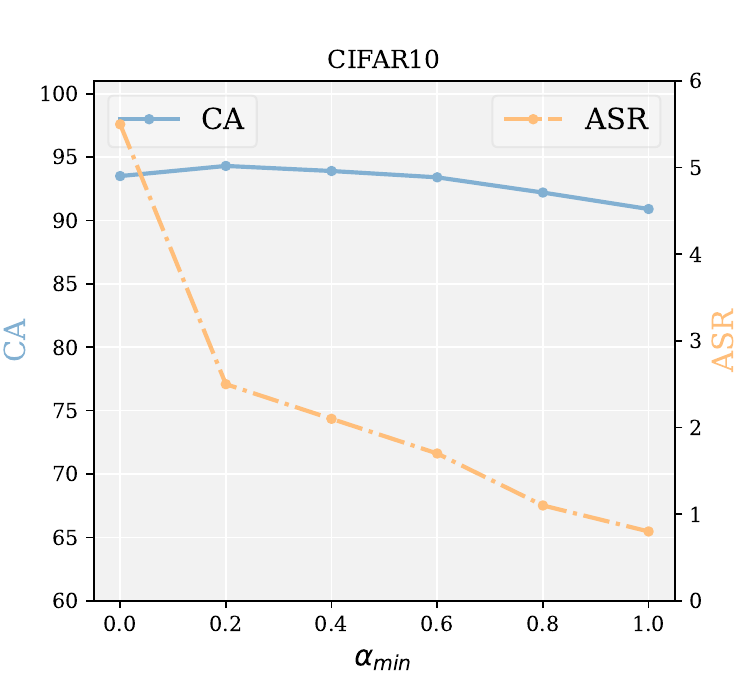}
    \includegraphics[width=0.49\columnwidth]{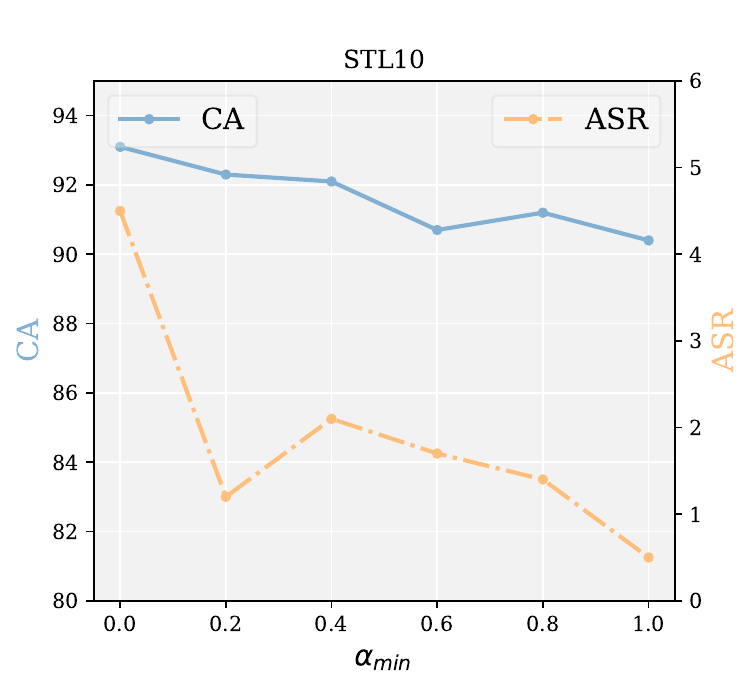}
    \caption{Sensitivity analysis on $\alpha_{\min}$.}
    \label{fig: alpha}
\end{figure}
\vspace{-5pt}

\section{Related Works}
\noindent{\textbf{Semi-supervised Learning.}}
Semi-supervised learning (SSL) is a well-established field featuring a wide range of approaches\cite{mclachlan1975iterative}. In this section, we focus on methods that adopt self-training paradigm, which represents the most mainstream techniques in modern SSL\cite{scudder1965probability}. The core concept involves treating the model's output probabilities as either soft or hard pseudo labels for unlabeled data\cite{lee2013pseudo, berthelot2019remixmatch, zhai2019s4l}. Additionally, consistency regularization is employed to ensure that predictions on perturbed versions of the unlabeled data remain the same\cite{xie2020unsupervised, xu2021dash, wang2022freematch, chen2023softmatch}. Moreover, there are studies that concentrate on enhancing the model's robustness in SSL\cite{guo2022robust, jia2023bidirectional, li2023iomatch, wan2024unlocking}. However, these papers, referred to as safe SSL, focus on learning from unlabeled data with distribution shifts or additional classes, excluding the consideration of poisoned data \cite{li2019towards, guo2020safe}.

\noindent{\textbf{Backdoor Attacks and Defenses.}}
A backdoor adversary aims to implant backdoor functionality into a target model. Recent studies \cite{connor2022rethinking, turner2019label} have highlighted that almost all contemporary SSL methods remain highly susceptible to certain specially designed clean label-type backdoor attacks. Shejwalkar's work, as one of the most influencing works in SSL backdoor attacks, has systematically identified characteristics that successful backdoor trigger should possess\cite{shejwalkar2023perils}. They also find that defense methods in supervised learning \cite{shokri2020bypassing,liu2023beating} become ineffective or challenging to be implemented. The primary reason is that these defense methods heavily rely on labeled data or some characteristics of learned features to detect poisoned samples or neutralize implanted backdoors\cite{dong2021black, tejankar2023defending, wang2024mm}. However, in SSL, the extremely limited number of labeled data points makes such approaches unfeasible. Specifically, common observations for backdoored model in supervised learning like activation clustering \cite{chen2018detecting}, loss divergence \cite{li2021anti} and large-margin logit\cite{MM-BD} no longer exists in attacked SSL models. These issues also reflect the urgency and difficulty of developing backdoor defense methods for SSL models.

\section{Conclusion}
In this study, we focus on protecting SSL algorithms from backdoor attacks. By analyzing the mechanics of existing successful attacks from a causal perspective, we introduce the first plug-in defense method for SSL, designed to filter, obstruct, and dilute these attacks through comprehensive data processing and label learning strategies. We further demonstrate the effectiveness of our proposed BI in enhancing backdoor defense effectiveness and preserving clean data accuracy, supported by extensive empirical evidence and theoretical validations. It's also worth mentioning that BI  does not require additional detection steps, making it more
efficient than most existing defense strategies.

\newpage

{
    \small
    \bibliographystyle{ieeenat_fullname}
    \bibliography{main}
}
\newpage
\appendix
\clearpage 
\section{Algorithm Description} \label{appendix: alg}
\begin{algorithm}[htb]
\caption{Training procedure of the proposed method}
\label{alg3}
\textbf{Input}: labeled batch $\mathcal{B}_l$, unlabeled batch $\mathcal{B}_u$;\\ 
\textbf{Parameter}: confidence threshold $\tau$, phase one iteration $t_1$, max iteration $t_{\max}$, Gaussian Filter radius $\gamma$;\\
\textbf{Output}: classifier $f(\cdot)$, feature extractor $g(\cdot)$ and model parameters $\Theta$;
\begin{algorithmic}[1] 
\STATE \textbf{Initialize $\Theta$}, $t=0$ and compute trigger mixup coefficient $\alpha$ by a cosine scheduler;
\WHILE{$t$ $\le$ $t_1$}
\STATE Implement the Gaussian Filter on labeled batch $\mathcal{B}_l$ and unlabeled batch $\mathcal{B}_u$;
\STATE Compute the probability $g(x_l)$, $g(x_u)$ and the output label $f(x_l)$ , $f(x_u)$ of the input data;
\STATE Generate complementary labels $\hat{\bar{y_u}}$ by Algorithm 1;
\STATE Update the transition matrix $\mathbf{Q}$ by Algorithm 2;
\STATE Compute the loss via Eq.\ref{loss1} with $\hat{\bar{y_u}}$ and $\mathbf{Q}$;
\STATE Update model parameters $\Theta$ via optimizer;
\ENDWHILE
\WHILE{$t_1<t$ $\le$ $t_{\max}$}
\STATE Implement the Gaussian Filter on labeled batch $\mathcal{B}_l$ and unlabeled batch $\mathcal{B}_u$;
\STATE Compute the probability $g(x_l)$, $g(x_u)$ and the output label $f(x_l)$ , $f(x_u)$ of the input data;
\STATE Compute the loss via Eq.\ref{loss2}
\STATE Update model parameters $\Theta$ via optimizer;
\ENDWHILE
\STATE \textbf{return} classifier $f(\cdot)$, feature extractor $g(\cdot)$ and model parameters $\Theta$;
\end{algorithmic}
\end{algorithm}

\begin{figure}[ht]
\centering
\includegraphics[width=1\columnwidth]{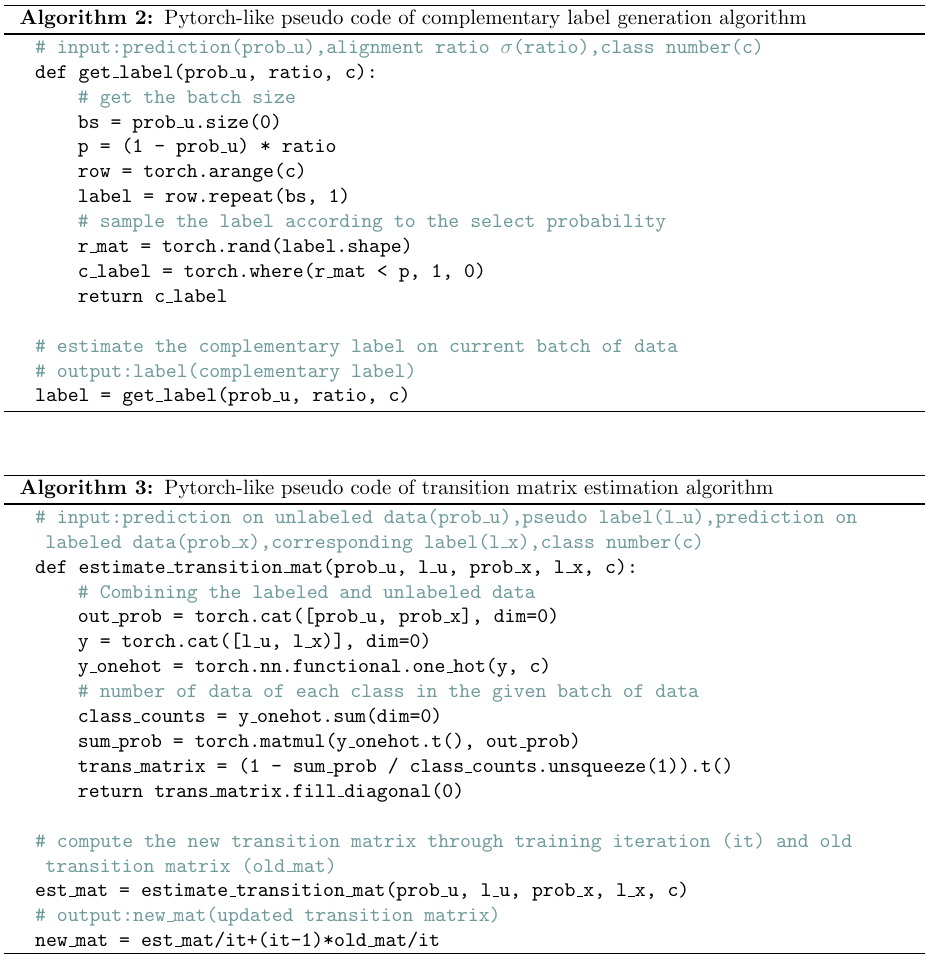}
\label{alg}
\end{figure}


In this section, we provide a comprehensive description of the training procedure for the proposed method, as outlined in Algorithm 1. Furthermore, we include PyTorch-like pseudocode for the generation of complementary labels and the estimation of the transition matrix, which were discussed in the "Backdoor Obstruction" section. For a more detailed implementation and the specific code, please refer to the supplementary materials provided.

\section{Detailed Experimental Setup} \label{setup}

\begin{figure*}[ht]
    \centering
    \includegraphics[width=0.24\textwidth]{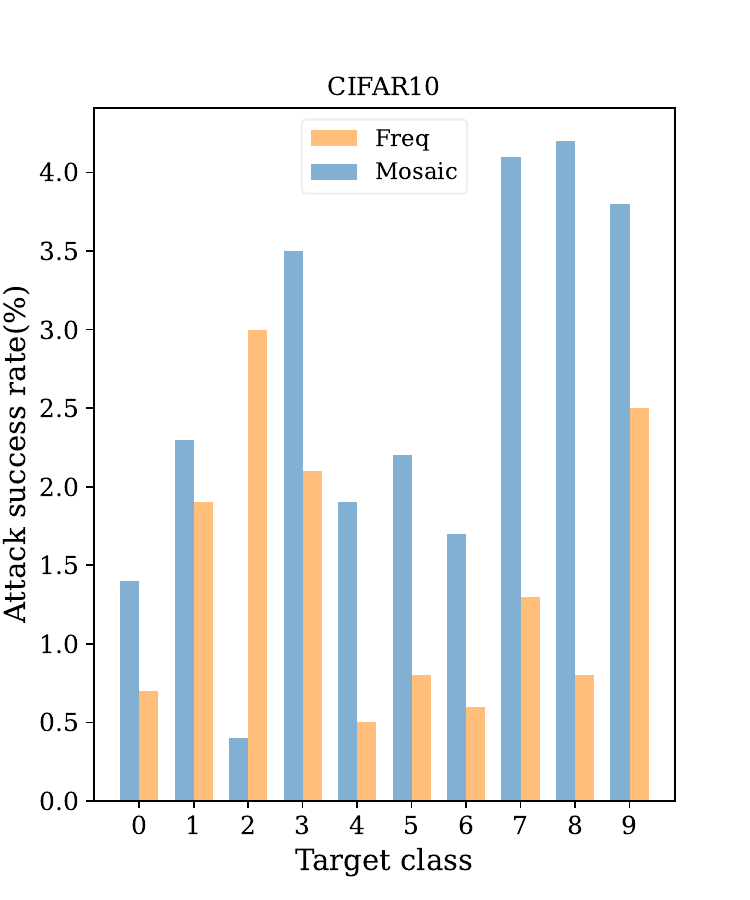}
    \includegraphics[width=0.24\textwidth]{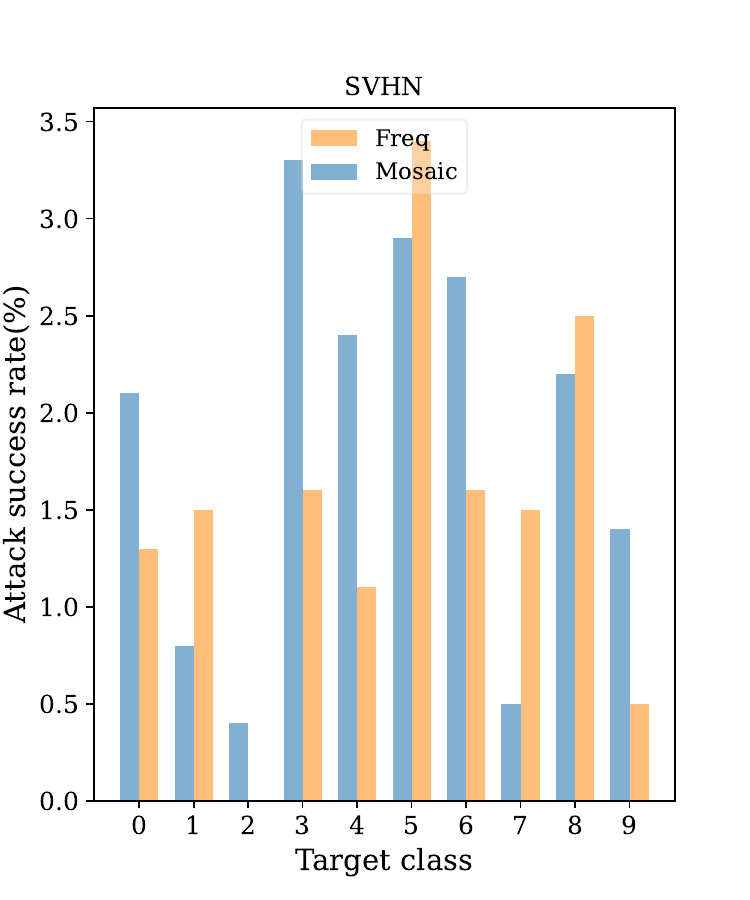}
    \includegraphics[width=0.24\textwidth]{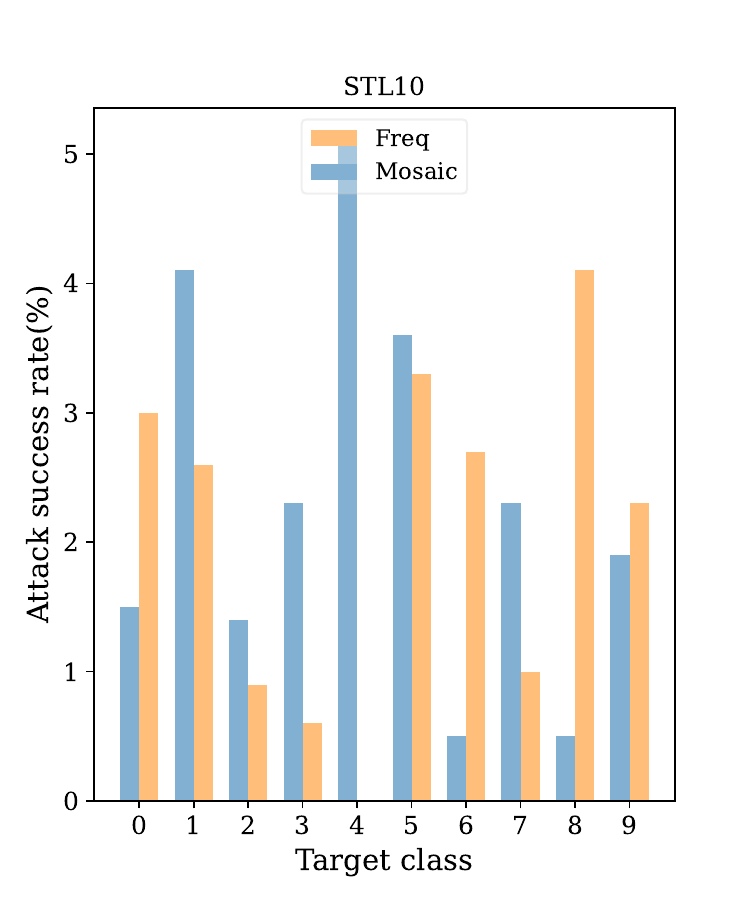}
    \includegraphics[width=0.24\textwidth]{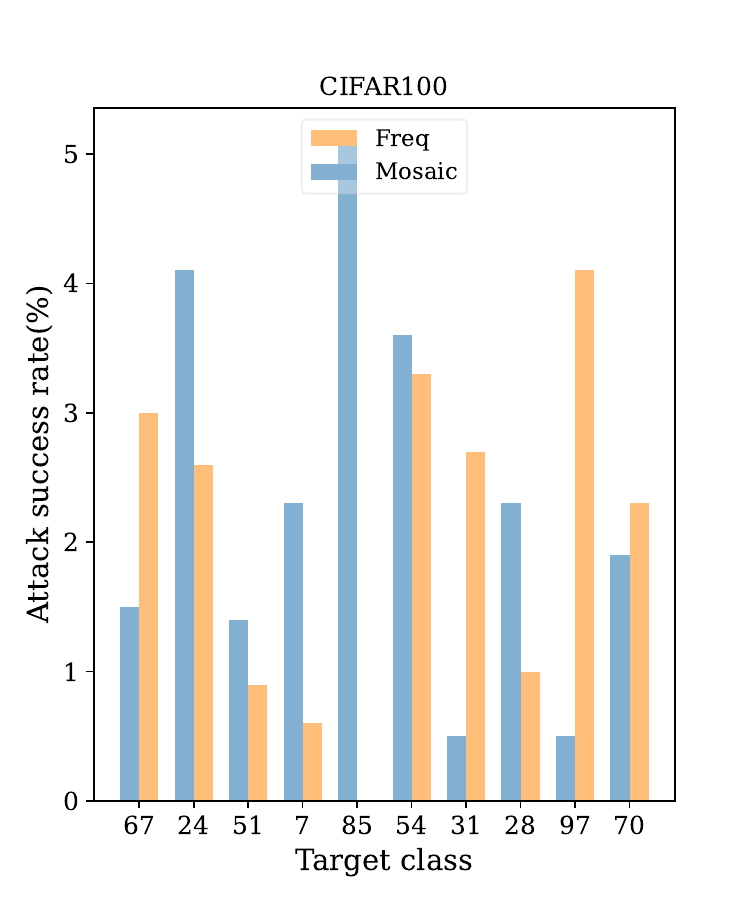}
    \caption{\textit{ASR} of our method across different target classes under Freq and Mosaic.}
    \label{target class}
\end{figure*}

\subsection{Datasets and model architectures} 
We evaluate our backdoor attacks using four datasets (CIFAR10, SVHN, CIFAR100, STL10) \cite{krizhevsky2009learning, coates2011analysis, netzer2011reading} commonly utilized to benchmark semi-supervised learning algorithms:
\begin{itemize}
    \item CIFAR10: This dataset is designed for a 10-class classification task and contains 60,000 RGB images, split into 50,000 for training and 10,000 for testing. Each image is 32 × 32 pixels with 3 channels. CIFAR10 is a class-balanced dataset, where each of the 10 classes contains exactly 6,000 images. For the semi-supervised learning models, we use 400 samples per class and employ a WideResNet architecture with a depth of 28, a widening factor of 2, and 1.47 million parameters.
    \item SVHN (Street View House Numbers): This dataset also supports a 10-class classification task and includes 73,257 images for training and 26,032 images for testing. Each image measures 32 × 32 pixels and has 3 channels. Unlike CIFAR10, SVHN is not class-balanced. The number of labeled samples per class used in SVHN is 10, and the same WideResNet architecture is applied.
    \item CIFAR100 is a dataset designed for a 100-class classification task, comprising 60,000 RGB images (50,000 for training and 10,000 for testing), each of size 32×32 and containing 3 channels. CIFAR100 is class-balanced, with each class evenly represented across the dataset. We selected CIFAR100 to evaluate our defense methods because it presents a significantly more complex challenge compared to both CIFAR10 and SVHN. For this task, the number of labeled samples per class used is 25 and we utilize a WideResNet model with a depth of 28 and a widening factor of 8, featuring 23.4 million parameters.
    \item STL10 is a dataset tailored for semi-supervised learning research, featuring a 10-class classification task. It includes 100,000 unlabeled images and 5,000 labeled images, maintaining class balance across the dataset. Each image is 96×96 pixels with 3 channels.  For this task, the number of labeled samples per class used is 100. Consistent with prior studies, we employ a similar 2-layer WideResNet architecture as used for the CIFAR10 and SVHN datasets.
\end{itemize}
\subsection{Details of the hyperparameters of experiments}
\subsubsection{Training hyperparameters.}
Initially, we train and assess other SSL (Semi-Supervised Learning) methods employing a unified codebase, as found in Wang et al. \cite{wang2022usb}, using their original hyperparameters. These parameters remain unaltered in benign settings without a backdoor adversary to maintain consistency. For fairness in comparison, we adhere to the protocol described by Shejwalkar et al. \cite{shejwalkar2023perils}, conducting experiments over 2,000,000 iterations. The results are presented as the median of 5 runs for CIFAR-10 and SVHN, 3 runs for STL-10, and a single run for CIFAR-100.

\subsubsection{Device.} All the experiments are implemented on NVIDIA RTX2080ti and RTX4090ti.

\begin{figure*}[ht]
    \centering
    \includegraphics[width=0.48\textwidth]{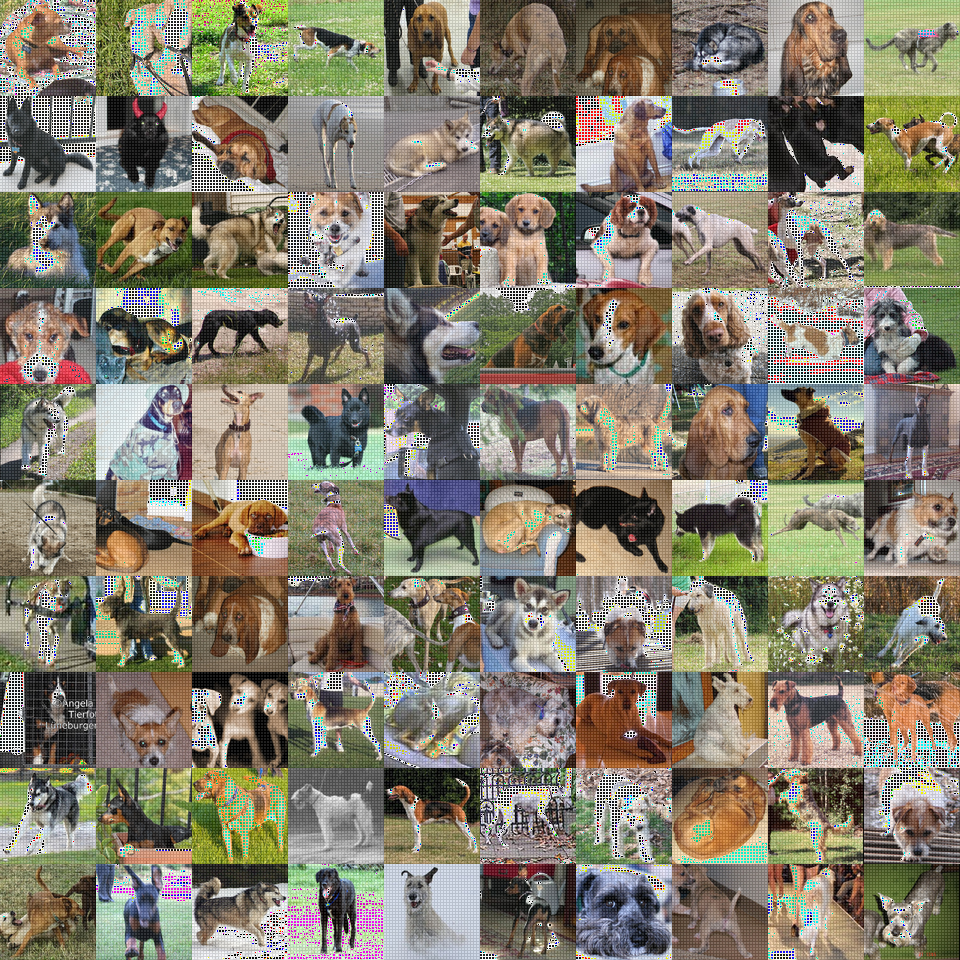}
    \includegraphics[width=0.48\textwidth]{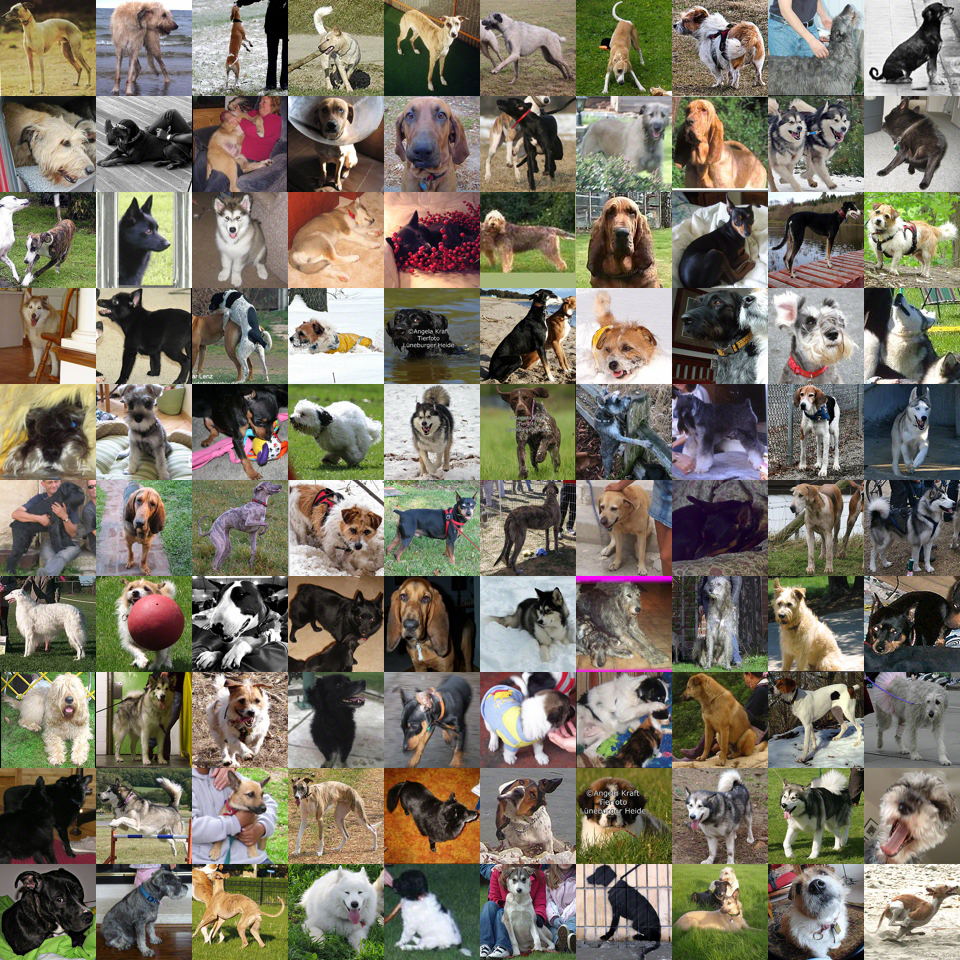}
    \caption{Visualization of the 100 poisoned images (dog as the target class) of two most successful SSL backdoor attacks: Mosaic on the \textbf{left} and Freq on the \textbf{right}.}
    \label{poison_vis_all}
\end{figure*}

\begin{figure}
    \centering
    \includegraphics[width=0.95\columnwidth]{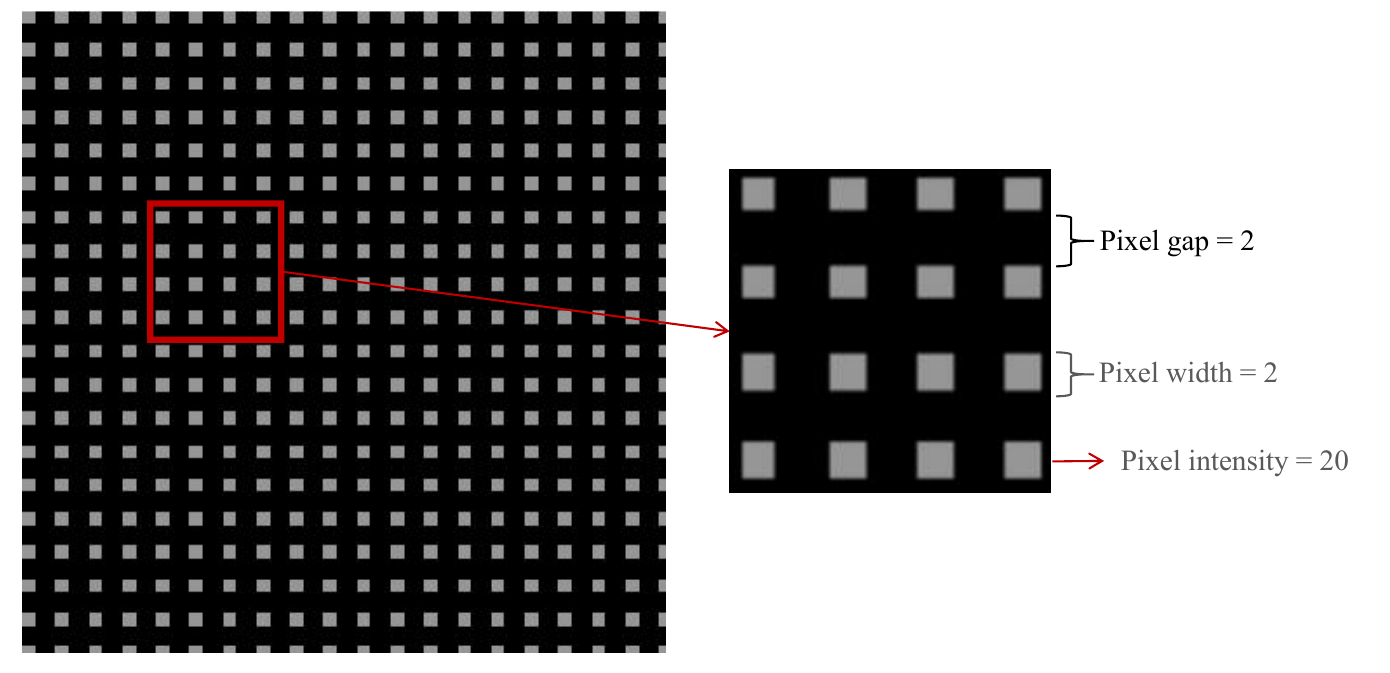}
    \caption{Visualization and the details of the backdoor triggers of Mosaic attack\cite{shejwalkar2023perils}.}
    \label{fig: trigger}
\end{figure}

\subsubsection{Attack hyperparameters.} \label{appdendix: attack setting}
For the baseline attacks including DeHiB \footnote{https://github.com/yanzhicong/DeHiB}, Narcissus \footnote{https://github.com/ruoxi-jia-group/Narcissus-backdoor-attack}, and Freq \footnote{https://github.com/meet-cjli/CTRL}, we utilize code directly provided by the original authors. For the clean-label variant of Badnets, we employ a 4-square trigger, setting the pixel intensity of all four squares to 255. Regarding the Mosaic attack, we apply the attack intensity specified in the original paper by Shejwalkar et al. \cite{shejwalkar2023perils}, setting the gap between each Mosaic attack pixel to 1 for CIFAR-10, CIFAR-100, and SVHN; and to 2 for STL-10. Additionally, due to the unavailability of the Mosaic attack code, we have re-implemented it according to the details provided in the paper and included it in our supplementary materials.As illustrated in Figure \ref{fig: trigger}, we adopt the pixel gap, pixel width, and pixel intensity settings for the backdoor trigger as described in \cite{shejwalkar2023perils}. It is important to note that for results other than Fixmatch and Mixmatch, we maintain these settings consistent with those in \cite{shejwalkar2023perils} to ensure a high Attack Success Rate (\textit{ASR}) for SSL methods without employing backdoor defense techniques.

Additionally, in Figure \ref{poison_vis_all}, we provide a visualization of the two most successful SSL backdoor attacks, Mosaic and Freq. Specifically, we illustrate 100 poisoned data samples with Mosaic-like triggers and frequency-based perturbations. It can be seen that, compared to Mosaic, Freq is more discreet, making it very hard for humans to distinguish between the poisoned and clean images.

\subsubsection{Defend hyperparameters.}
We have discussed the selection of hyperparameters in the main text. Across various datasets, we set the radius of the Gaussian Filter to 1 and the trigger mix-up coefficient to 0.2 to ensure a fair comparison in Table 1 and Table 2. In our experiments, we observed that the radius of the Gaussian Filter could be reduced for the "Freq" attack compared to the "Mosaic" attack to maintain better accuracy on clean data. Generally, these hyperparameters modulate the intensity of the defense strategy. When dealing with stronger attacks, it is advisable to implement more robust defenses, and conversely, less intense defenses may suffice for weaker attacks. For other defensive strategies in our baseline, we provide a brief description of the defenses and discuss the results; for detailed information on these defenses, please refer to the respective original works. For standard fine-tuning, we fine-tune the backdoored model using available benign labeled data. Specifically, we use the labeled training data from the SSL algorithm and adjust the learning rate hyperparameter to achieve optimal results. We aim to maintain the \textit{CA} of the final fine-tuned model within 10\% of the \textit{CA} achieved without any defense. For fine-pruning, we initially prune the parameters of the last convolutional layer of the backdoored model that are not activated by benign data. Subsequently, we fine-tune the pruned model using the available benign labeled data. For NAD, we begin by fine-tuning a backdoored model to create a teacher model with relatively lower \textit{ASR}. Following this, NAD trains the original backdoored model (the student) to align the activations of various convolutional layers between the teacher and the student.

\subsection{Contention between the trigger pattern and the natural feature patterns in the early training stage.} \label{contention}

\begin{figure}[ht] 
    \centering
    \label{poison_vis}
    \includegraphics[width=0.48\columnwidth]{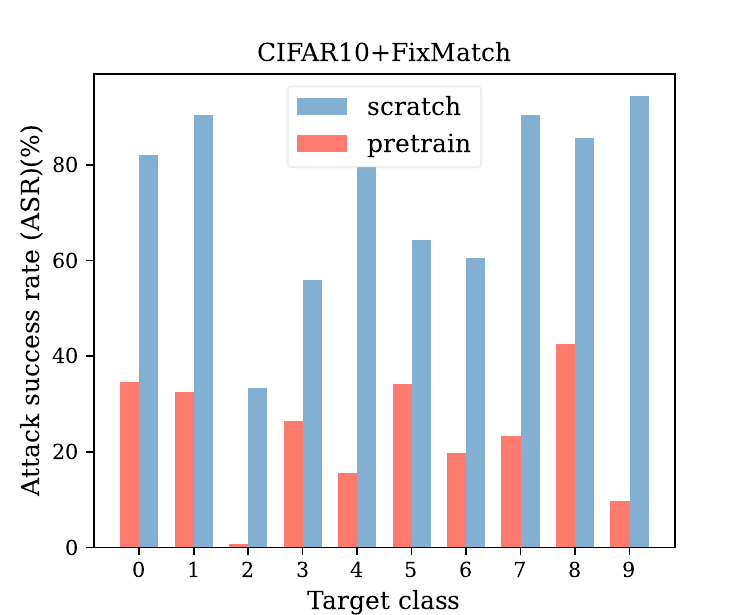}
    \includegraphics[width=0.48\columnwidth]{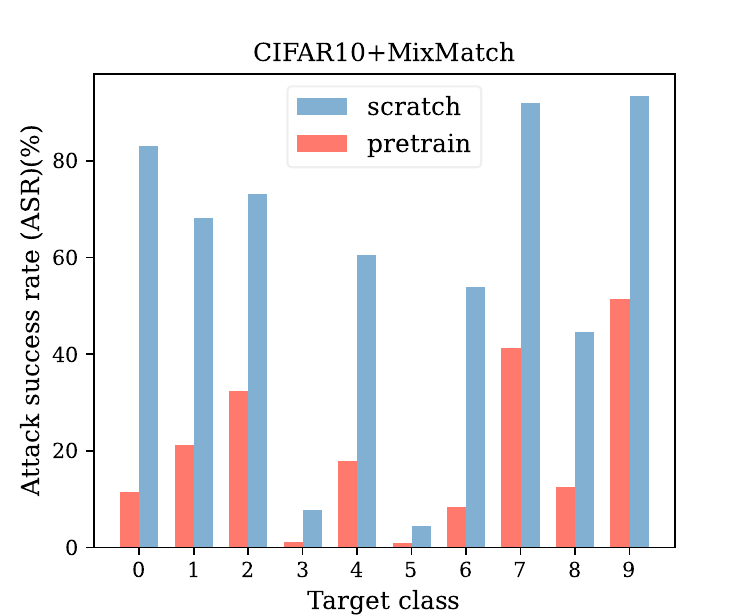}
    \caption{ASR of Mosaic attack for different target classes.}
    \label{hist}
\end{figure}

During our pilot experiments, we observed a notable contention between the trigger pattern and the natural feature patterns early in the training process. Specifically, we trained two models on the same poisoned dataset (CIFAR10 with 100 poisoned images): one model was trained from scratch, while the other was initialized with parameters pre-trained on the clean ImageNet dataset. Figure \ref{hist} demonstrates that compared to training from scratch, utilizing a clean pre-trained model significantly reduces the risk of the model succumbing to an attack, especially evident in the first 100,000 iterations of training. To some extent, the model preferentially models spurious correlations driven by the trigger pattern alongside the causal effects driven by natural feature patterns. This observation actually inspires us to obstruct backdoor attacks in the early training stage, as once the causal effect-driven natural feature patterns are solidly established, introducing spurious correlations becomes much more challenging.

\begin{table*}[htb] 
\caption{\textit{CA} and \textit{ASR} of the proposed Backdoor Invalidator (BI) with different number of labels per class ($n_c$).} \label{tab: label_num}
    \centering
    \begin{subtable}{\textwidth}
    \centering
    \small
    \begin{tabular} {|c|c|c|c|c|c|c|c|c|c|c|c|c|c|}
      \hline
      \multirow{6}{*}{\rotatebox{90}{\textbf{MOSAIC}}} & \multirow{2}{*}{$|\mathcal{D}_u|$} & \multicolumn{2}{c|}{CIFAR10} & \multicolumn{2}{c|}{SVHN} & \multicolumn{2}{c|}{STL10} & \multicolumn{2}{c|}{CIFAR100} \\ \cline{3-10}
      && CA $\uparrow$ & ASR $\downarrow$ & CA $\uparrow$ & ASR $\downarrow$ & CA $\uparrow$ & ASR $\downarrow$ & CA $\uparrow$  &ASR $\downarrow$ \\ \cline{2-10}
      &$n_c=4$ & 78.1 & 3.2 & 92.8 & 0.7 & 60.9 & 1.3 & 49.6 & 0.8 \\ \cline{2-10}
      &$n_c=25$ & 86.4 & 1.9 & 94.2& 0.2 & 83.0 & 2.1 & 65.4 & 3.2 \\ \cline{2-10}
      &$n_c=100$ & 92.7 & 0.6 & 95.1 & 0.5 & 92.0 & 1.2 & 66.9 & 2.1 \\ \cline{2-10}
      &$n_c=400$ & 93.4 & 2.5 & 97.3 & 0.6 & 94.7 & 1.1 & 77.9 & 0.0 \\ \hline
    \end{tabular}
    \end{subtable}

    \begin{subtable}{\textwidth}
    \centering
    \small
    \begin{tabular} {|c|c|c|c|c|c|c|c|c|c|c|c|c|c|}
      \hline
      \multirow{6}{*}{\rotatebox{90}{\textbf{FREQ}}} & \multirow{2}{*}{$|\mathcal{D}_u|$} & \multicolumn{2}{c|}{CIFAR10} & \multicolumn{2}{c|}{SVHN} & \multicolumn{2}{c|}{STL10} & \multicolumn{2}{c|}{CIFAR100} \\ \cline{3-10}
      && CA $\uparrow$ & ASR $\downarrow$ & CA $\uparrow$ & ASR $\downarrow$ & CA $\uparrow$ & ASR $\downarrow$ & CA $\uparrow$  &ASR $\downarrow$ \\ \cline{2-10}
      &$n_c=4$ & 80.6 & 0.2 & 93.3 & 0.4 & 64.1 & 0.6 & 53.2 & 0.5 \\ \cline{2-10}
      &$n_c=25$ & 90.3 & 0.8 & 94.0& 0.9 & 91.4 & 1.0 & 64.8 & 0.9 \\ \cline{2-10}
      &$n_c=100$ & 93.0 & 0.1 & 94.7 & 1.2 & 92.4 & 0.2 & 67.6 & 0.4 \\ \cline{2-10}
      &$n_c=400$ & 93.8 & 0.7 & 98.1 & 0.9 & 95.1 & 0.4& 80.4 & 0.0 \\ \hline
    \end{tabular}
    \end{subtable}
\end{table*}

\section{More detailed experimental results} \label{appendix: exp}
In this section, we present some additional experimental results, including the effects of different backdoor target classes, and the impact of varying the size (number) of labeled and poisoned data. 
\subsection{Influence of different target classes.}\label{appendix: target class}
As depicted in Figure \ref{target class}, \textit{ASR} varies significantly across different target classes in both the "Freq" and "Mosaic" backdoor attacks. This variation is consistent with observations discussed in the foundational literature on backdoor attacks. However, while these phenomena are evident, the underlying reasons remain unclear. We plan to explore these aspects in our future work.

\subsection{Influence of the quantity of labels.}\label{appedix: label_num}

Subsequently, we explore the influence of the number of labels on the performance of our proposed method, BI. For consistency across evaluations, we use the same number of labels per class for different datasets. Specifically, we employ 40, 250, 1000, and 4000 labels for CIFAR10 and SVHN, and 400, 2500, 10000, and 40000 labels for CIFAR100. As illustrated in Table \ref{tab: label_num}, the\textit{CA} of BI decreases sharply as the number of labels decreases, showing a performance gap compared to state-of-the-art Semi-Supervised Learning (SSL) methods like FlexMatch and SemiReward \cite{li2023semireward}. Part of the reason is that our method relies heavily on labeled data for models to capture the feature patterns necessary to counteract potential trigger patterns in unlabeled data. However, this strategy inevitably limits the model's learning capability when the number of labeled data is scarce.

\begin{table*}[htb] 
\caption{\textit{CA} and \textit{ASR} of the proposed Backdoor Invalidator (BI) with different number of poisoned data.} \label{tab: poison_num}
    \centering
    \begin{subtable}{\textwidth}
    \centering
    \small
    \begin{tabular} {|c|c|c|c|c|c|c|c|c|c|c|c|c|c|c|}
      \hline
      \multirow{6}{*}{\rotatebox{90}{\textbf{MOSAIC}}} & \multirow{2}{*}{Poison ratio: $p_c$} & \multicolumn{2}{c|}{CIFAR10} & \multicolumn{2}{c|}{SVHN} & \multicolumn{2}{c|}{STL10} & \multirow{2}{*}{Poison ratio: $p_c$} &  \multicolumn{2}{c|}{CIFAR100} \\ \cline{3-8} \cline{10-11}
      && CA $\uparrow$ & ASR $\downarrow$ & CA $\uparrow$ & ASR $\downarrow$ & CA $\uparrow$ & ASR $\downarrow$ && CA $\uparrow$  &ASR $\downarrow$ \\ \cline{2-11}
      &$p_c=2\%$ & 93.4 & 2.5 & 95.1 & 0.5 & 92.0 & 1.2 &$p_c=20\%$ & 65.4 & 3.2 \\ \cline{2-11}
      &$p_c=10\%$ & 93.2 & 4.2 & 94.8 & 1.1 & 91.7 & 2.1 &$p_c=100\%$ & 61.6 & 21.3 \\ \cline{2-11}
      &$p_c=50\%$ & 92.5 & 7.6 & 93.3 & 4.2 & 91.1 & 5.5 & -- & -- & -- \\ \hline
    \end{tabular}
    \end{subtable}

    \begin{subtable}{\textwidth}
    \centering
    \small
    \begin{tabular} {|c|c|c|c|c|c|c|c|c|c|c|c|c|c|c|}
      \hline
      \multirow{6}{*}{\rotatebox{90}{\textbf{FREQ}}} & \multirow{2}{*}{Poison ratio: $p_c$} & \multicolumn{2}{c|}{CIFAR10} & \multicolumn{2}{c|}{SVHN} & \multicolumn{2}{c|}{STL10} & \multirow{2}{*}{Poison ratio: $p_c$} &  \multicolumn{2}{c|}{CIFAR100} \\ \cline{3-8} \cline{10-11}
      && CA $\uparrow$ & ASR $\downarrow$ & CA $\uparrow$ & ASR $\downarrow$ & CA $\uparrow$ & ASR $\downarrow$ && CA $\uparrow$  &ASR $\downarrow$ \\ \cline{2-11}
      &$p_c=2\%$ & 93.8 & 0.7 & 94.7 & 1.2 & 91.6 & 0.2 &$p_c=20\%$ & 67.6 & 0.4\\ \cline{2-11}
      &$p_c=10\%$ & 92.9 & 1.8 & 94.9 & 0.9 & 92.7 & 0.7 &$p_c=100\%$ & 64.3 & 14.5 \\ \cline{2-11}
      &$p_c=50\%$ & 94.3 & 4.5 & 95.1 & 2.6 & 92.0 & 5.8 & -- & -- & -- \\ \hline
    \end{tabular}
    \end{subtable}
\end{table*}

\subsection{Influence of the number of poisoned data.}\label{appendix: poison_num}

Additionally, we examine how our proposed Backdoor Invalidator (BI) method performs against varying quantities of poisoned data. As illustrated in Table \ref{tab: poison_num}, BI consistently achieves satisfactory results across different poison ratios. Notably, the poison ratio $p$ refers to the percentage of poisoned data in the entire dataset and poison ratio $p_c$ refers to the percentage of poisoned data in the target class. Given that successful attacks in SSL often involve clean-label attacks, the ratios of 0.2\%, 1.0\%, and 5.0\% correspond to 2\%, 10\%, and 50\% of the data in the target class during training for CIFAR10, SVHN, and STL10, respectively, and 20\%, 100\%, and N/A for CIFAR100.

\subsection{Performance when BI is integrated with other SSL methods.} \label{appendix: other ssl methods}

In the main text, due to space constraints, we only integrated BI with FixMatch \cite{sohn2020fixmatch} and FlexMatch \cite{zhang2021flexmatch}. Here, we provide additional evaluations of the proposed plug-in backdoor defense methods with MixMatch \cite{berthelot2019mixmatch} and SemiReward \cite{li2023semireward}. As shown in Table \ref{tab:more_results}, BI consistently achieves low ASR while maintaining performance on clean data across most datasets. For SemiReward, the performance degradation is somewhat more significant. We assume this is because the substitution from consistency loss to complementary label learning in the first training stage hampers the reward function in the original algorithm.

\section{Limitations and future works.} \label{limitaion}

In scenarios where both labeled and unlabeled datasets are compromised, combining our method with existing defense strategies could offer a robust solution. However, in our experiments, integrating existing backdoor defense strategies for supervised learning proved challenging. The typical scarcity of labeled data makes it difficult for the defense methods we tested, such as ABL and FP, to effectively detect poisoned data or prune the implanted backdoor. We acknowledge this as a key limitation of our method and plan to address it in future work.

\begin{table*}[htb] 
\caption{The attack success rate (ASR \%)  and the clean accuracy (CA \%) of another 2 SSL algorithms with our proposed method against 5 representative backdoor attacks.} \label{tab:more_results}
\centering
\begin{subtable}{\textwidth}
\centering
\small
\begin{tabular} {|c|c|c|c|c|c|c|c|c|c|c|c|c|c|c|}
  \hline
  \multirow{6}{*}{\rotatebox{90}{\textbf{CIFAR10}}} & \multirow{2}{*}{Algorithm}  & \multicolumn{2}{c|}{CL-Badnets} & \multicolumn{2}{c|}{Narcissus} & \multicolumn{2}{c|}{DeHiB*} & \multicolumn{2}{c|}{Mosaic} & \multicolumn{2}{c|}{Freq}\\ \cline{3-12}
  
  && CA $\uparrow$ & ASR $\downarrow$ & CA $\uparrow$ & ASR $\downarrow$ & CA $\uparrow$ & ASR $\downarrow$ & CA $\uparrow$ & ASR $\downarrow$ & CA $\uparrow$ & ASR $\downarrow$\\ \cline{2-12}
  &Mixmatch & 93.4 & 16.8 & 93.8 & 2.2 & 93.2 & 22.0 & 94.2 & 96.8 & 93.4 & 83.7 \\ \cline{2-12}
  &Mixmatch w/ BI & 91.2 & 0.2 & 90.9 &0.1 & 91.1 & 0.4 & 89.4 & 1.1 & 90.4 & 2.1 \\\cline{2-12}
  &SemiReward & 95.0 & 7.2 & 94.1 & 5.8 & 95.1 & 10.9 & 93.7 & 89.6 & 95.3 & 47.2 \\  \cline{2-12}
  &SemiReward w/ BI& 93.1 & 0.1 & 92.8 & 0.9 & 92.6 & 1.3 & 90.5 & 4.3 & 91.9 & 2.6 \\ \hline
\end{tabular}
\end{subtable}

\begin{subtable}{\textwidth}
\centering
\small
\begin{tabular} {|c|c|c|c|c|c|c|c|c|c|c|c|c|c|c|}
  \hline
  \multirow{6}{*}{\rotatebox{90}{\textbf{\textbf{SVHN}}}} & \multirow{2}{*}{Algorithm}  & \multicolumn{2}{c|}{CL-Badnets} & \multicolumn{2}{c|}{Narcissus} & \multicolumn{2}{c|}{DeHiB*} & \multicolumn{2}{c|}{Mosaic} & \multicolumn{2}{c|}{Freq}\\ \cline{3-12}
  
  && CA $\uparrow$ & ASR $\downarrow$ & CA $\uparrow$ & ASR $\downarrow$ & CA $\uparrow$ & ASR $\downarrow$ & CA $\uparrow$ & ASR $\downarrow$ & CA $\uparrow$ & ASR $\downarrow$\\ \cline{2-12}
  &Mixmatch  & 93.5 & 5.0 & 93.2 & 0.0 & 94.2 & 2.5 & 92.9 & 87.7 & 93.3 & 90.3 \\ \cline{2-12}
  &Mixmatch w/ BI  & 92.1 & 1.1 & 92.0 & 0.0 & 91.0 & 0.4 & 91.9 & 3.1 & 92.4 & 0.8 \\ \cline{2-12}
  &SemiReward & 92.8 & 6.3& 94.5 & 10.2 & 92.6 & 3.1 & 93.6 & 78.6 & 94.2 & 58.0 \\  \cline{2-12}
  &SemiReward w/ BI & 91.5 & 0.2 & 93.0 & 0.4 & 90.1 & 0.7 & 92.5 & 0.9 & 90.6 & 1.3 \\ \hline
\end{tabular}
\end{subtable}

\begin{subtable}{\textwidth}
\centering
\small
\begin{tabular} {|c|c|c|c|c|c|c|c|c|c|c|c|c|c|c||c|c|}
  \hline
  \multirow{6}{*}{\rotatebox{90}{\textbf{STL10}}} & \multirow{2}{*}{Algorithm} & \multicolumn{2}{c|}{CL-Badnets} & \multicolumn{2}{c|}{Narcissus} & \multicolumn{2}{c|}{DeHiB} & \multicolumn{2}{c|}{Mosaic} & \multicolumn{2}{c|}{Freq}\\ \cline{3-12}
  
  && CA $\uparrow$ & ASR $\downarrow$ & CA $\uparrow$ & ASR $\downarrow$ & CA $\uparrow$ & ASR $\downarrow$ & CA $\uparrow$ & ASR $\downarrow$ & CA $\uparrow$ & ASR $\downarrow$\\ \cline{2-12}
  
  &Mixmatch & 90.3 & 11.6 & 89.6 & 2.0 & 88.8 & 1.1 & 88.9 & 87.5 &90.9 & 86.4\\ \cline{2-12}
  &Mixmatch w/ BI & 90.1 & 0.8 & 87.7& 1.5 & 89.2 & 0.4 & 87.5 & 2.5 &89.4 & 1.3 \\ \cline{2-12}

  &SemiReward & 92.0 & 16.8 & 91.7 & 3.2 & 90.4 & 4.5 & 91.9 & 68.7 & 92.5 & 72.6 \\  \cline{2-12}
  &SemiReward w/ BI & 90.5 & 1.0 & 90.9 & 0.0 & 88.4 & 0.3 & 87.5 & 1.4 & 89.2 & 0.9 \\ \hline
\end{tabular}
\end{subtable}

\begin{subtable}{\textwidth}
\centering
\small
\begin{tabular} {|c|c|c|c|c|c|c|c|c|c|c|c|c|c|c||c|c|}
  \hline
  \multirow{6}{*}{\rotatebox{90}{\textbf{CIFAR100}}} & \multirow{2}{*}{Algorithm} & \multicolumn{2}{c|}{CL-Badnets} & \multicolumn{2}{c|}{Narcissus} & \multicolumn{2}{c|}{DeHiB*} & \multicolumn{2}{c|}{Mosaic} & \multicolumn{2}{c|}{Freq}\\ \cline{3-12}
  && CA $\uparrow$ & ASR $\downarrow$ & CA $\uparrow$ & ASR $\downarrow$ & CA $\uparrow$ & ASR $\downarrow$ & CA $\uparrow$ & ASR $\downarrow$ & CA $\uparrow$ & ASR $\downarrow$\\ \cline{2-12}
  &Mixmatch &65.7 & 29.4 & 70.0 & 1.9 & 67.5 & 9.4 & 71.6 & 92.8 & 66.9 & 87.4 \\ \cline{2-12}
  &Mixmatch w/ BI &62.2 & 0.2 & 67.8 & 0.0 & 65.4 & 0.3 & 63.8 & 2.4 & 64.1 & 0.5 \\\cline{2-12}
  &SemiReward &70.8 & 14.2 & 71.5 & 5.6 & 70.3 & 1.2 & 72.0 & 96.3 & 73.5 & 74.9 \\ \cline{2-12}
  &SemiReward w/ BI & 65.8 & 0.6 & 66.0 & 0.0 & 62.6 & 1.0 & 61.9 & 6.1 & 63.7 & 1.6 \\\hline
\end{tabular}
\end{subtable}
\end{table*}

\section{Proof}\label{appdendix: proof}
\subsection{Proof of Theorem 1}
\begin{proof}
According to Assumption 1 and based on the modified loss function, when learning from examples with complementary labels, we also have 
\[q_i^*(x) = P(\bar{y}=i|x), \forall i \in [c].\]

Let $\mathbf{v}(x)=[P(y=1|x),\cdots,P(y=c|x)]$ and $\bar{\mathbf{v}}(x)=[P(\bar{y}=1|x),\cdots,P(\bar{y}=c|x)]$. We have 
\begin{equation}
\bar{\mathbf{v}}(x) = \mathbf{Q}^\top \mathbf{v}(x),
\end{equation}
which further ensures
\begin{equation}
\mathbf{q}^*(x) = \mathbf{Q}^\top \mathbf{v}(x) = \mathbf{Q}^\top \mathbf{g}^*(x).
\end{equation}

If the transition matrix $\mathbf{Q}$ is invertible, then we find the optimal $\mathbf{g}^*(x)=\mathbf{v}(x)$, which means that the minimizer $\bar{f}^*$ derived by complementary learning coincides with the optimal classifier of semi-supervised learning. 
\end{proof}

\subsection{Proof of Theorem 2}
Before providing the detailed proof of Theorem 2, we first provide some useful lemmas.
\begin{lemma} \cite{yu2018learning} \label{lemma1}
    Let $\bar{\ell}(f(x),\bar{y}) = -\log\left(\frac{\sum_{k=1}^c Q_{ki} \exp(h_k(x))}{\sum_{k=1}^c \exp(h_k(x))} \right)$, where $y^i=0$ and suppose that $h_i(x) \in \mathcal{H}, \forall i \in [c]$, we have $\mathfrak{R}_{m_{i}}(\bar{\ell}\circ\mathcal{F}) \leq c \mathfrak{R}_{m_{i}}(\mathcal{H})$.
\end{lemma}

In order to prove Lemma 1, we need the loss function $\bar{\ell}(f(x),\bar{y})$ to be Lipschitz continous with respect to $h_i(x)$, which can be proved by the following lemma,

\begin{proof}
Recall that 
\begin{equation}
    \bar{\ell}(f(x),\bar{y}) = -\log\left(\frac{\sum_{k=1}^c Q_{ki}\exp(h_k(x))}{\sum_{k=1}^c \exp(h_k(x))}\right)
\end{equation}
Take the derivative of $\bar{\ell}(f(x),\bar{y}=i)$ with respect to $h_j(x)$, we have:
\begin{equation} \label{derivative1}
\begin{aligned}
&\frac{\partial \bar{\ell}(f(x),\bar{y}=i)}{\partial h_j(x)} = -\frac{Q_{ji}\exp(h_j(x))}{\sum_{k=1}^c Q_{ki}\exp(h_k(x))} \\ &+ \frac{\exp(h_j(x))}{\sum_{k=1}^c \exp(h_k(x))}.
\end{aligned}
\end{equation}

According to Eq.(\ref{derivative1}), it is easy to conclude that $-1 \leq \frac{\partial \bar{\ell}(f(x),\bar{y}=i)}{\partial h_j(x)} \leq 1$, which also indicates that the loss function is 1-Lipschitz with respect to $h_j(x), \forall j \in [c]$. 

Now we are ready to prove Lemma 1. Since the softmax function preserve the rank of its inputs, $f(x)=\arg\max_{i \in [c]} g_i(x) = \arg\max_{i \in [c]} h_i(x)$. We thus have
\begin{equation}
\begin{aligned}
&\mathfrak{R}_{n_{i}}(\bar{\ell}\circ\mathcal{F})\\
&=\mathbb{E}\left[\sup_{f\in \mathcal{F}}\frac{1}{n_i}\sum_{j=1}^{n_i}\sigma_j\bar{\ell}(f(x_j),\bar{y}_j=i)\right] \\
&= \mathbb{E}\left[\sup_{\arg\max\{h_1(x),\cdots,h_c(x)\} }\frac{1}{n_i}\sum_{j=1}^{n_i}\sigma_j\bar{\ell}(f(x_j),\bar{y}_j=i)\right] \\
&= \mathbb{E}\left[\sup_{\max\{h_1(x),\cdots,h_c(x)\} }\frac{1}{n_i}\sum_{j=1}^{n_i}\sigma_j\bar{\ell}(f(x_j),\bar{y}_j=i)\right] \\
&\leq \mathbb{E}\left[ \sum_{k=1}^c \sup_{h_k(x) }\frac{1}{n_i}\sum_{j=1}^{n_i}\sigma_j\bar{\ell}(f(x_j),\bar{y}_j=i)\right] \\
&= \mathbb{E}\left[ \sum_{k=1}^c \sup_{h_k(x) }\frac{1}{n_i}\sum_{j=1}^{n_i}\sigma_j\log\left(\frac{\sum_{m=1}^c Q_{mi} \exp(h_m(x))}{\sum_{m=1}^c \exp(h_m(x))} \right)\right] \\ 
&= \sum_{k=1}^c \mathbb{E}\left[ \sup_{h_k(x) }\frac{1}{n_i}\sum_{j=1}^{n_i}\sigma_j\log\left(\frac{\sum_{m=1}^c Q_{mi} \exp(h_m(x))}{\sum_{m=1}^c \exp(h_m(x))} \right)\right].
\end{aligned}
\end{equation}
Here, the argument $f\in\mathcal{F}$ of $\sup$ function indicates that $f$ is chosen from the function space $\mathcal{F}$. The function space $\mathcal{F}$ is actually determined by the function space of $\mathbf{h}$ due to the fact that $f = \arg\max\{g_1(x),\cdots,g_c(x)\}=\arg\max \{h_1(x),\cdots,h_c(x)\}$. Thus, the argument of $\sup$ function can be changed to $\arg\max{h_1(x),\cdots,h_c(x)}$ in the second equality. Since $\arg\max \{h_1(x),\cdots,h_c(x)\}$ and $\max\{h_1(x),\cdots,h_c(x)\}$ give the same constraint on $h_i(x), \forall i \in[c]$, the argument is changed to $\max\{h_1(x),\cdots,h_c(x)\}$ in the third equality.

According to Talagrand’s contraction theorem\cite{talagrand1995concentration}, we have
\begin{equation} \label{radl}
\begin{aligned}
\mathfrak{R}_{m_{i}}(\bar{\ell}\circ\mathcal{F}) &\leq \sum_{k=1}^c \mathbb{E}\left[ \sup_{h_k(x) }\frac{1}{n_i}\sum_{j=1}^{n_i}\sigma_j h_k(x)\right]\\
&= \sum_{k=1}^c \mathfrak{R}_{m_{i}}(\mathcal{H}) \\
&= c\mathfrak{R}_{m_{i}}(\mathcal{H}),
\end{aligned}
\end{equation}
The proof is completed.
\end{proof}

\begin{lemma} \label{lemma2}
    Let $\ell(f(x),y) = -\log\left(\frac{\sum_{k=1}^c Q_{ki} \exp(h_k(x))}{\sum_{k=1}^c \exp(h_k(x))} \right)$, where $y^i=1$ and suppose that $h_i(x) \in \mathcal{H}, \forall i \in [c]$, we have $\mathfrak{R}_{n_{i}}(\bar{\ell}\circ\mathcal{F}) \leq c \mathfrak{R}_{n_{i}}(\mathcal{H})$.
\end{lemma}

Similar to the proof of Lemma 1, we also need the loss function$\ell(f(x), y)$ to be Lipschitz continous with respect to $h_i(x)$ which can be proved as follows:
\begin{proof}
Recall that 
\begin{equation}
    \ell(f(x),y) = -\log\left(\frac{\exp(h_k(x))}{\sum_{k=1}^c \exp(h_k(x))}\right)
\end{equation}
Take the derivative of $\ell(f(x),{y}=i)$ with respect to $h_j(x)$, we have:
\begin{equation} \label{derivative2}
\begin{aligned}
&\frac{\partial \bar{\ell}(f(x),{y}=i)}{\partial h_j(x)} = \left\{
\begin{aligned}
    \frac{\exp(h_j(x))}{\sum_{k=1}^c \exp(h_k(x))} &, i \neq j \\
    -\frac{\sum_{k=1, k\neq i}^c\exp(h_k(x))}{\sum_{k=1}^c \exp(h_k(x))} &, i=j
\end{aligned}
\right.
\end{aligned}
\end{equation}
According to Eq.(\ref{derivative2}), it is also easy to conclude that $-1 \leq \frac{\partial \bar{\ell}(f(x),\bar{y}=i)}{\partial h_j(x)} \leq 1$, which also indicates that the loss function is 1-Lipschitz with respect to $h_j(x), \forall j \in [c]$. Similar to the proof of Lemma2, we also have 
\begin{equation} \label{radu}
\begin{aligned}
\mathfrak{R}_{m_{i}}({\ell}\circ\mathcal{F}) &\leq \sum_{k=1}^c \mathbb{E}\left[ \sup_{h_k(x) }\frac{1}{n_i}\sum_{j=1}^{n_i}\sigma_j h_k(x)\right]\\
&= \sum_{k=1}^c \mathfrak{R}_{m_{i}}(\mathcal{H}) \\
&= c\mathfrak{R}_{m_{i}}(\mathcal{H}),
\end{aligned}
\end{equation}
according to the Talagrand’s contraction \cite{talagrand1995concentration}.
\end{proof}

The above two Lemmas help us unify the hypothesis space of the loss on labeled and unlabeled data on $\mathcal{H}$. Next we try to upper bound the estimation error of the pseudo labels during the training of unlabeled data.
\begin{lemma}
    Suppose the loss function $\bar{\ell}(\cdot)$ on unlabeled data be upper bounded by $M_2$. For some $\epsilon>0$, if $\sum^m_{i=1}\sum^c_{k=1}|\hat{y}^{ik}_u-y^{ik}_u|/m\leq \epsilon$, we have:
    \begin{equation}
        |\hat{R}^{'} (f)-\hat{R}(f)|\leq M_2\epsilon.
    \end{equation}
    where $y^{i}_u$ represents the true label of unlabeled data $x_u^i$ and $\hat{y}^{i}_u$ is the estimated pseudo label.
\end{lemma}
\begin{proof}
    Without loss of generality, we assume that \( \epsilon \) represents the largest pseudo-labeling error, defined as \( \epsilon = \max \left( \frac{1}{m} \sum_{i=1}^m \sum_{k=1}^c |\hat{y}^{ik}_u - y^{ik}_u| \right) \). We can partition this largest pseudo-labeling error into two components:
    \begin{equation}
        \begin{aligned}
            \epsilon_1 = \frac{1}{m}\sum_{i=1}^m \sum_{k=1}^c \mathbb{I}(y_u^{ik}\neq 0 \land \hat{y}_u^{ik}= 0) \\
            \epsilon_2 = \frac{1}{m}\sum_{i=1}^m \sum_{k=1}^c \mathbb{I}(y_u^{ik}= 0 \land \hat{y}_u^{ik} \neq 0)
        \end{aligned}
    \end{equation}
    where \( \epsilon_1 \) and \( \epsilon_2 \) respectively represent the error due to incorrect labels and the error due to missing labels. We then establish the following propositions, which provide the upper and lower bounds for the estimated pseudo-labeling error. Firstly, we prove its upper bound:
    \begin{equation}
        \begin{aligned}
            \hat{R}^{'}_{u}(f)& =\frac{1}{m} \sum_{i=1}^m \sum_{k=1}^c \mathbb{I}(\hat{y}_u^{ik}=0) \bar{\ell}(g_k(x^i_u))\\
            &\leq \frac{1}{m} \sum_{i=1}^m \sum_{k=1}^c \mathbb{I}(y_u^{ik}\neq 0 \land \hat{y}_u^{ik}= 0) \bar{\ell}(g_k(x^i_u)) + \\
            &\frac{1}{m} \sum_{i=1}^m \sum_{k=1}^c \mathbb{I}(y_u^{ik} = 0) \bar{\ell}(g_k(x^i_u)) \\
            &\leq M_2\epsilon_1 + \hat{R}_{u}(f)
        \end{aligned}
    \end{equation}
    Then, we prove the lower bound:
    \begin{equation}
        \begin{aligned}
            \hat{R}_{u}(f)& =\frac{1}{m} \sum_{i=1}^m \sum_{k=1}^c \mathbb{I}({y}_u^{ik}=0) \bar{\ell}(g_k(x^i_u))\\
            &\leq \frac{1}{m} \sum_{i=1}^m \sum_{k=1}^c \mathbb{I}(y_u^{ik}= 0 \land \hat{y}_u^{ik} \neq 0) \bar{\ell}(g_k(x^i_u)) + \\
            &\frac{1}{m} \sum_{i=1}^m \sum_{k=1}^c \mathbb{I}(\hat{y}_u^{ik}=0) \bar{\ell}(g_k(x^i_u)) \\
            &\leq M_2\epsilon_2 + \hat{R}^{'}_{u}(f)
        \end{aligned}
    \end{equation}
    By combining two sides, we can complete the proof:
    \begin{equation}\label{est_error}
        |\hat{R}^{'} (f)-\hat{R}(f)|\leq M_2 \max(\epsilon_1, \epsilon_2)\leq M_2\epsilon.
    \end{equation}
\end{proof}

Now, we give the proof of Theorem 2 in the main text, let us first reclaim it as follows.
\begin{theorem}
    Suppose $\bar{\pi}_k$ and $\pi_k$ are given. Let the loss function $\ell(\cdot)$ on labeled and loss function $\bar{\ell}(\cdot)$ on unlabeled data be upper bounded respectively by $M_1$ and $M_2$. For some $\epsilon>0$, if $\sum^m_{i=1}\sum^c_{k=1}|\hat{y}^{ik}_u-y^{ik}_u|/m\leq \epsilon$. Then, for any $\delta>0$, with the probability $1-c\delta$:
    \begin{equation}
    \begin{aligned}
        &\Tilde{R}(\hat{f}^{'})- \Tilde{R}(f^*) \leq \sum_{k=1}^c \Bigg( 4c \pi_k \mathfrak{R}_{n_{k}}(\mathcal{H}) + 4c \bar{\pi}_k \mathfrak{R}_{m_{k}}(\mathcal{H})  \\
        &+ 2\pi_k M_1\sqrt{\frac{\log{1/\delta}}{2n_{k}}} + 2\bar{\pi}_k M_2\sqrt{\frac{\log{1/\delta}}{2m_{k}}} \Bigg) + 2M_2\epsilon,
    \end{aligned}
    \end{equation}
    where $y^{i}_u$ represents the true label of unlabeled data $x_u^i$ and $\hat{y}^{i}_u$ is the estimated pseudo label; $\mathfrak{R}_{n}(\mathcal{H})=\mathbb{E}\left[ \sup_{h_k(x) }\frac{1}{n}\sum_{j=1}^{n}\sigma_j h_k(x)\right]$is the Rademacher complexity and $\{\sigma_1,\cdots,\sigma_n\}$ are Rademacher variables uniformly distributed from $\{-1,1\}$.
\end{theorem}
\begin{proof}
The convergence rates of generalization bounds of multi-class learning are at most $O(c^2/\sqrt{n})$ with respect to $c$ and $n$ \cite{ishida2017learning,mohri2018foundations}. To reduce the dependence on $c$ of our derived convergence rate, we rewrite ${R_l}(f)$ and ${R_u}(f)$ as follows:

\begin{equation}
\begin{aligned}
&{R_l}(f) =\int_{x} \sum_{i=1}^c P({y}=i) P(x|{y}=i) {\ell}(f(x),{y}=i) dx \\
&=\sum_{i=1}^c P({y}=i) \int_{x} P(x|{y}=i)  {\ell}(f(x),{y}=i) dx \\
&=\sum_{i=1}^c {\pi}_i {R}_{l}^i(f),
\end{aligned}
\end{equation}
\begin{equation}
\begin{aligned}
&{R_u}(f) =\int_{x} \sum_{i=1}^c P(\bar{y}=i) P(x|\bar{y}=i) \bar{\ell}(f(x),\bar{y}=i) dx \\
&=\sum_{i=1}^c P(\bar{y}=i) \int_{x} P(x|\bar{y}=i)  \bar{\ell}(f(x),\bar{y}=i) dx \\
&=\sum_{i=1}^c \bar{\pi}_i {R}_{u}^i(f),
\end{aligned}
\end{equation}

where ${R}_{u}^i(f) = \mathbb{E}_{x\sim P(x|\bar{y}=i)} \bar{\ell}(f(x),\bar{y}=i)$ and ${R}_{l}^i(f) = \mathbb{E}_{x\sim P(x|{y}=i)} {\ell}(f(x),{y}=i)$. Additionally, we denote the class prior of being labeled (true label and complementary label) as $\bar{\pi}_i=P(\bar{y}=i)$ and ${\pi}_i=P({y}=i)$. 

Then, we show an upper bound for the estimation error of our method. This upper bound illustrates a convergence rate for the classier learned with our proposed pseudo complementary labels to the optimal one learned with true labels.
    \begin{equation}
    \begin{aligned}
        &\Tilde{R}(\hat{f}^{'})- \Tilde{R}(f^*) 
        = \Tilde{R}(\hat{f}^{'})-\hat{\Tilde{R}}(\hat{f}^{'}) + \hat{\Tilde{R}}(\hat{f}^{'}) - \hat{\Tilde{R}}^{'}(\hat{f}^{'})\\
        &+ \hat{\Tilde{R}}^{'}(\hat{f}^{'}) -\hat{\Tilde{R}}^{'}(f^*)+ \hat{\Tilde{R}}^{'}(f^*) - \hat{\Tilde{R}}(f^*) + \hat{\Tilde{R}}(f^*) - \Tilde{R}(f^*)\\
        &\leq 2\sup_{f\in\mathcal{F}}|\Tilde{R}(f)-\hat{\Tilde{R}}(f)| + 2\sup_{f\in\mathcal{F}}|\hat{\Tilde{R}}(f)-\hat{\Tilde{R}}^{'}(f)|\\
        &= 2\sup_{f\in\mathcal{F}}|R_l(f)-\hat{R_l}(f)| + 2\sup_{f\in\mathcal{F}}|R_u(f)-\hat{R_u}(f)| \\
        &+ 2\sup_{f\in\mathcal{F}}|\hat{R}_u(f)-\hat{R}_u^{'}(f)|\\
        &\leq 2\sum_{i=1}^c \bar{\pi}_i\sup_{f\in \mathcal{F}} |{R}^{i}_u(f)-\hat{R}^{i}_u(f)| + 2\sup_{f\in\mathcal{F}}|\hat{R}_u(f)-\hat{R}_u^{'}(f)| \\
        &+2\sum_{i=1}^c {\pi}_i\sup_{f\in \mathcal{F}} |{R}^{i}_l(f)-\hat{R}^{i}_l(f)|
    \end{aligned}
\end{equation}
where the first inequality holds because $ \hat{\Tilde{R}}^{'}(\hat{f}^{'})-\hat{\Tilde{R}}^{'}(f^*)<0$ and the error in the last line is the sum of generalization error and pseudo labeling estimation error.

Next, let us respectively upper bound the generalization error on labeled and unlabeled data.
Suppose ${\pi}_i=P({y}=i)$ is given, let the loss function on labeled data be upper bounded by $M_1$. Then, for any $\delta>0$, with the probability $1-c\delta$, we have
\begin{equation}\label{gen_l}
\begin{aligned}
{R}_l(\hat{f}^{'})&-{R}_l({f}^*) \leq 2\sup_{f\in \mathcal{F}}|{R}_l(f)-{R}_l(f)| \\
&\leq2\sum_{i=1}^c {\pi}_i\sup_{f\in \mathcal{F}} |{R}_l(f)-\hat{R}_l(f)|\\
&\leq 2\sum_{i=1}^c {\pi}_i\left(2\mathfrak{R}_{n_{i}}( {\ell}\circ\mathcal{F})+M_1\sqrt{\frac{\log{1/\delta}}{2n_{i}}}\right)\\
&= \sum_{i=1}^c \left( 4 {\pi}_i \mathfrak{R}_{n_{i}}( {\ell}\circ\mathcal{F}) + 2{\pi}_iM_1\sqrt{\frac{\log{1/\delta}}{2n_{i}}} \right)\\
&\leq \sum_{i=1}^c \left( 4c {\pi}_i \mathfrak{R}_{n_{i}}( \mathcal{H}) + 2{\pi}_iM_1\sqrt{\frac{\log{1/\delta}}{2n_{i}}} \right),
\end{aligned}
\end{equation}
$\mathfrak{R}_{n_{i}}({\ell}\circ\mathcal{F})=\mathbb{E}\left[\sup_{f\in \mathcal{F}}\frac{1}{n_i}\sum_{j=1}^{n_i}\sigma_j\bar{\ell}(f(X_j),\bar{Y}_j=i)\right]$ is the corresponding Rademacher complexity. The second line is the results in \cite{bartlett2002rademacher} and the fifth line is due to the results in Eq.\ref{radl}.

Similarly, we can derive that:
Suppose $\bar{\pi}_i=P(\bar{y}=i)$ is given, let the loss function on unlabeled data be upper bounded by $M_2$. Then, for any $\delta>0$, with the probability $1-c\delta$, we have
\begin{equation}\label{gen_u}
\begin{aligned}
{R}_u(\hat{f}^{'})&-{R}_u({f}^*) \leq 2\sup_{f\in \mathcal{F}}|{R}_u(f)-{R}_u(f)| \\
&\leq2\sum_{i=1}^c \bar{\pi}_i\sup_{f\in \mathcal{F}} |{R}_u(f)-\hat{R}_u(f)|\\
&\leq 2\sum_{i=1}^c \bar{\pi}_i\left(2\mathfrak{R}_{m_{i}}( {\bar{\ell}}\circ\mathcal{F})+M_2\sqrt{\frac{\log{1/\delta}}{2m_{i}}}\right)\\
&= \sum_{i=1}^c \left( 4 \bar{\pi}_i \mathfrak{R}_{m_{i}}( {\bar{\ell}}\circ\mathcal{F}) + 2{\pi}_iM_2\sqrt{\frac{\log{1/\delta}}{2m_{i}}} \right)\\
&\leq \sum_{i=1}^c \left( 4c \bar{\pi}_i \mathfrak{R}_{m_{i}}( \mathcal{H}) + 2{\pi}_iM_2\sqrt{\frac{\log{1/\delta}}{2m_{i}}} \right),
\end{aligned}
\end{equation}
$\mathfrak{R}_{m_{i}}({\ell}\circ\mathcal{F})=\mathbb{E}\left[\sup_{f\in \mathcal{F}}\frac{1}{m_i}\sum_{j=1}^{m_i}\sigma_j\bar{\ell}(f(X_j),\bar{Y}_j=i)\right]$ is the corresponding Rademacher complexity.The second line is the results in \cite{bartlett2002rademacher} and the fifth line is due to the results in Eq.\ref{radu}.

Then, combining the results in Eq.\ref{gen_l}, Eq.\ref{gen_u} and Eq.\ref{est_error}, we can have:

Suppose $\bar{\pi}_k$ and $\pi_k$ are given. Let the loss function $\ell(\cdot)$ on labeled and loss function $\bar{\ell}(\cdot)$ on unlabeled data be upper bounded respectively by $M_1$ and $M_2$. For some $\epsilon>0$, if $\sum^m_{i=1}\sum^c_{k=1}|\hat{y}^{ik}_u-y^{ik}_u|/m\leq \epsilon$. Then, for any $\delta>0$, with the probability $1-c\delta$:
\begin{equation}
    \begin{aligned}
        &\Tilde{R}(\hat{f}^{'})- \Tilde{R}(f^*) \leq \\
        &2\sum_{i=1}^c \bar{\pi}_i\sup_{f\in \mathcal{F}} |{R}^{i}_u(f)-\hat{R}^{i}_u(f)| + 2\sup_{f\in\mathcal{F}}|\hat{R}_u(f)-\hat{R}_u^{'}(f)| \\
        &+2\sum_{i=1}^c {\pi}_i\sup_{f\in \mathcal{F}} |{R}^{i}_l(f)-\hat{R}^{i}_l(f)|\\
        &\leq \sum_{k=1}^c \Bigg( 4 \pi_k \mathfrak{R}_{n_{k}}(\ell\circ\mathcal{F}) + 4 \bar{\pi}_k \mathfrak{R}_{m_{k}}(\bar{\ell}\circ\mathcal{F})  \\
        &+ 2\pi_k M_1\sqrt{\frac{\log{1/\delta}}{2n_{k}}} + 2\bar{\pi}_k M_2\sqrt{\frac{\log{1/\delta}}{2m_{k}}} \Bigg) + 2M_2\epsilon\\
        &\leq \sum_{k=1}^c \Bigg( 4c \pi_k \mathfrak{R}_{n_{k}}(\mathcal{H}) + 4c \bar{\pi}_k \mathfrak{R}_{m_{k}}(\mathcal{H})  \\
        &+ 2\pi_k M_1\sqrt{\frac{\log{1/\delta}}{2n_{k}}} + 2\bar{\pi}_k M_2\sqrt{\frac{\log{1/\delta}}{2m_{k}}} \Bigg) + 2M_2\epsilon.
    \end{aligned}
\end{equation}
\end{proof}
which completes the proof.

\end{document}